\documentclass{custom}

\graphicspath{{figures/}{images/}}

\usepackage{amsmath,amsfonts,amssymb,bm,dsfont}

\DeclareMathOperator{\Law}{Law}

\def\eqref#1{(\ref{#1})}

\DeclareMathAlphabet{\mathsfit}{\encodingdefault}{\sfdefault}{m}{sl}
\SetMathAlphabet{\mathsfit}{bold}{\encodingdefault}{\sfdefault}{bx}{n}

\def\1{{\mathds{1}}}

\newcommand{\E}{\mathbb{E}} 

\newcommand{\R}{\mathbb{R}}

\newcommand{\KL}{D_{\mathrm{KL}}}

\DeclareMathOperator*{\argmax}{arg\,max}

\usepackage{mathtools}
\usepackage{array}
\usepackage{xcolor}
\usepackage{textcomp}
\usepackage{url}

\crefformat{equation}{(#2#1#3)}
\crefrangeformat{equation}{(#3#1#4) to (#5#2#6)}
\crefmultiformat{equation}{(#2#1#3)}{ and (#2#1#3)}{, (#2#1#3)}{ and (#2#1#3)}

\newcommand{\mathd}{\mathrm{d}}

\renewcommand\contributionformat[2][]{{\footnotesize\color{slate}\if\relax#1\relax\else$^{#1}$\fi#2}}

\title{Reinforce Adjoint Matching}
\subtitle{Scaling RL Post-Training of Diffusion and Flow-Matching Models}
\author[1]{Andreas Bergmeister}
\author[1,2]{Stefanie Jegelka}
\author[3]{Nikolas Nüsken}
\authorbreak
\author[4,\dagger]{Carles Domingo-Enrich}
\author[5,\dagger]{Jakiw Pidstrigach}
\affiliation[1]{TU Munich, MCML}
\affiliation[2]{MIT CSAIL}
\affiliation[3]{King's College London}
\affiliation[4]{Microsoft Research New England}
\affiliation[5]{\mbox{University of Oxford}}
\contribution{$^\dagger$Joint last authors}
\metadata[Code]{\url{https://github.com/AndreasBergmeister/ram}}
\abstract{%
Diffusion and flow-matching models scale because pretraining is supervised regression: a clean sample is noised analytically, and a model regresses against a closed-form target. 
RL post-training aligns the model with a reward. In image generation, this makes samples compose objects correctly, render text legibly, and match human preferences. 
Existing methods rely on costly SDE rollouts, reward gradients, or surrogate losses, sacrificing pretraining's regression structure.
We show that the structure extends to RL post-training.
Under KL-regularized reward maximization, the optimal generative process tilts the clean-endpoint distribution towards samples with higher reward and leaves the noising law unchanged. Combining this with the adjoint-matching optimality condition and a REINFORCE identity, we derive Reinforce Adjoint Matching (RAM): a consistency loss that corrects the pretraining target with the reward.
At each step, we draw a clean endpoint from the current model, evaluate its reward, noise it as in pretraining, and regress. No SDE rollouts, backward adjoint sweeps, or reward gradients are required. Like the pretraining objective, RAM is simple and scales.
On Stable Diffusion 3.5M, RAM achieves the highest reward on composability, text rendering, and human preference, reaching Flow-GRPO's peak reward in up to $50\times$ fewer training steps.
}

\newcommand{\preprinttitlefigureheaderpair}{%
  \makebox[0.497\linewidth][c]{\small\textsf{\textbf{SD3.5M}}}%
  \hspace{0.004\linewidth}%
  \makebox[0.497\linewidth][c]{\small\textsf{\textbf{RAM}}}%
}

\newcommand{\preprinttitlefigurepair}[3]{%
  \includegraphics[width=0.49\linewidth]{#1}%
  \hspace{0.003\linewidth}%
  \includegraphics[width=0.49\linewidth]{#2}\par
  \vspace{0.03em}
  {\scriptsize \itshape #3\par}%
}

\newcommand{\preprinttitlefiguretaskcolumn}[8]{%
  \begin{minipage}[t]{0.315\textwidth}
    \centering
    {\small\sffamily\bfseries #1\par}%
    \vspace{0.02em}
    \includegraphics[width=0.96\linewidth]{#2}\par
    \vspace{0.22em}
    \preprinttitlefigureheaderpair\par
    \vspace{0.12em}
    \preprinttitlefigurepair{#3}{#4}{#5}
    \vspace{0.12em}
    \preprinttitlefigurepair{#6}{#7}{#8}
  \end{minipage}%
}

\newcommand{\preprinttitlefigurecontent}{%
  \begin{minipage}[c]{\textwidth}
    \centering
    \preprinttitlefiguretaskcolumn
      {Compositional Image Generation}
      {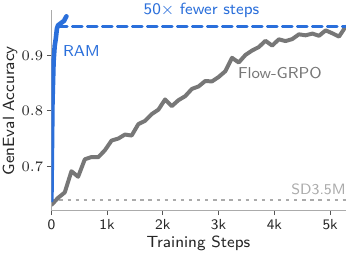}
      {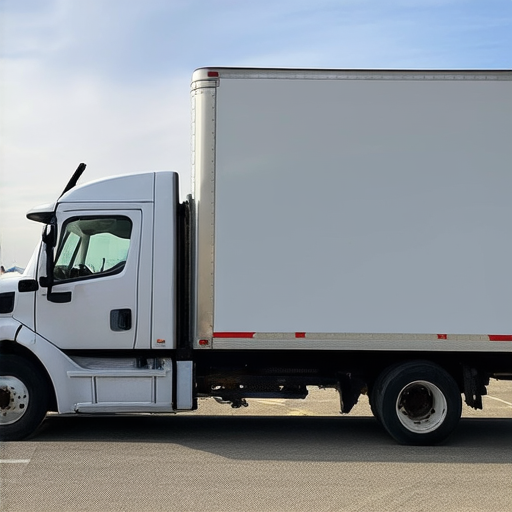}{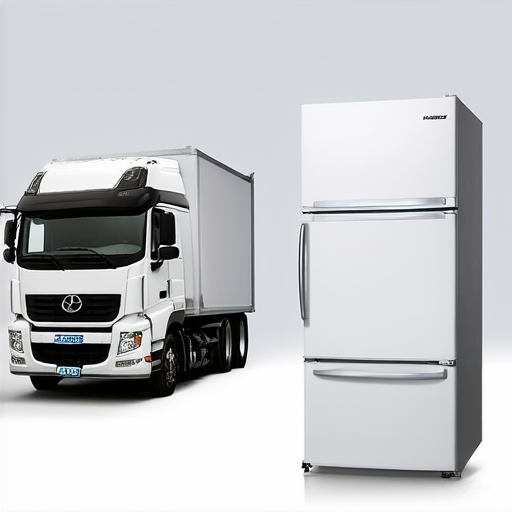}{truck left of refrigerator}
      {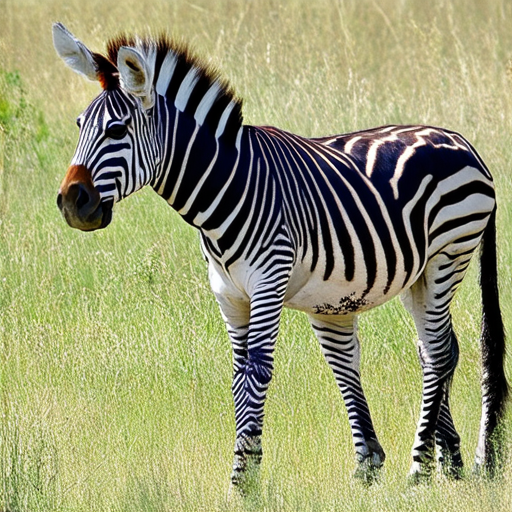}{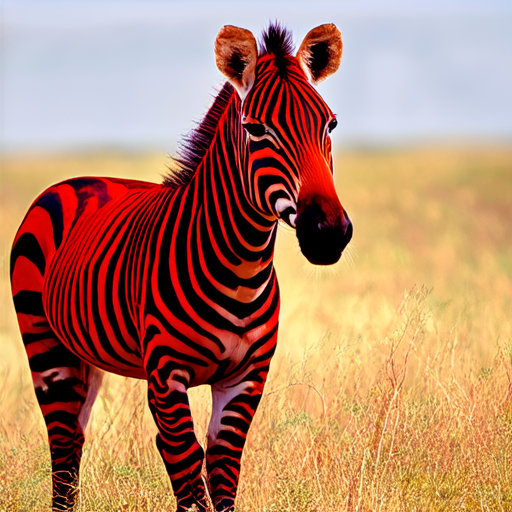}{red zebra}%
    \hfill
    \preprinttitlefiguretaskcolumn
      {Visual Text Rendering}
      {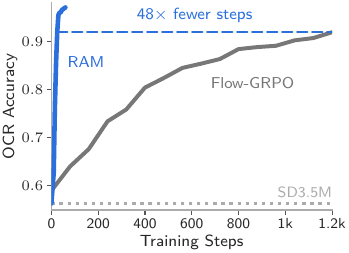}
      {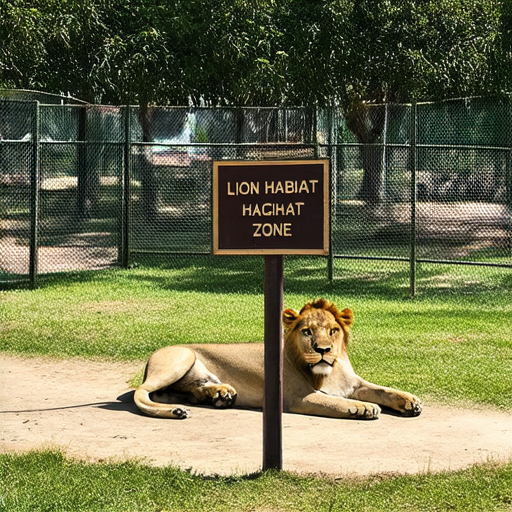}{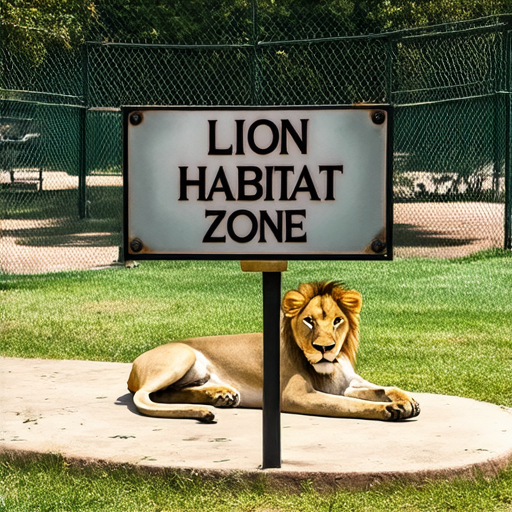}{zoo sign: ``Lion Habitat Zone''}
      {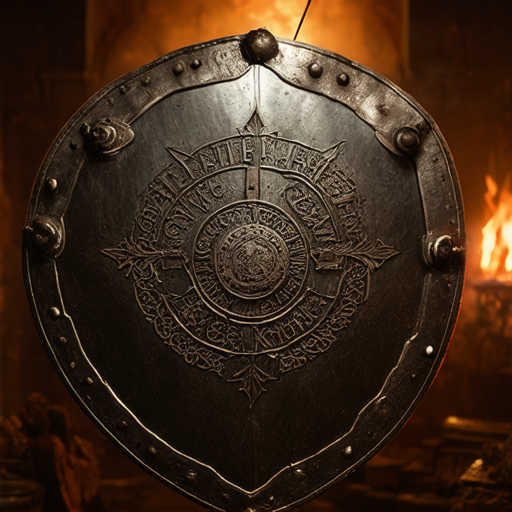}{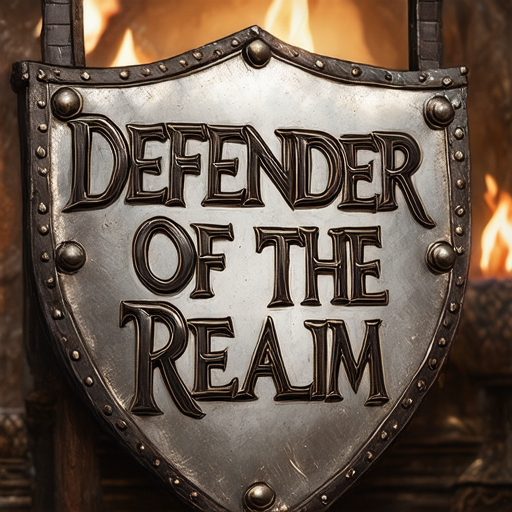}{medieval shield: ``Defender Of The Realm''}%
    \hfill
    \preprinttitlefiguretaskcolumn
      {Human Preference Alignment}
      {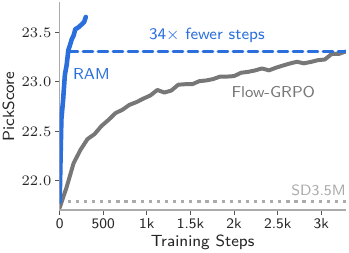}
      {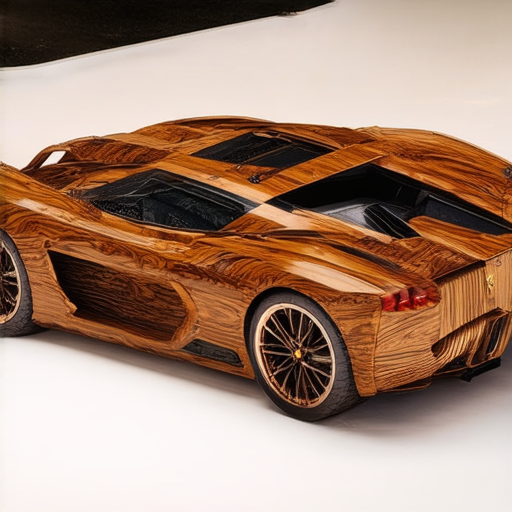}{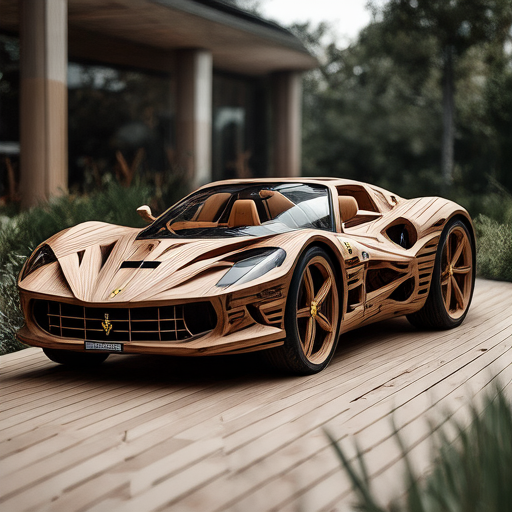}{wooden Ferrari}
      {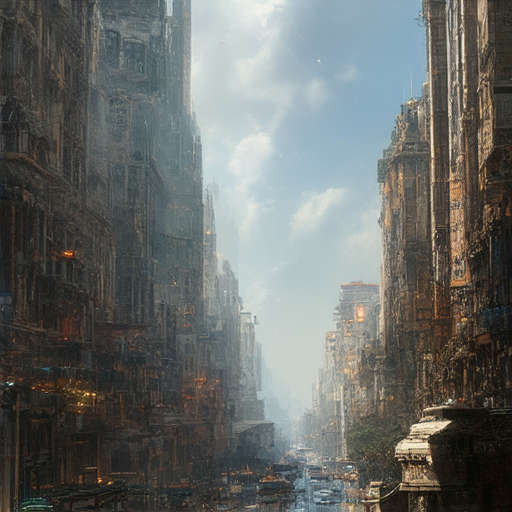}{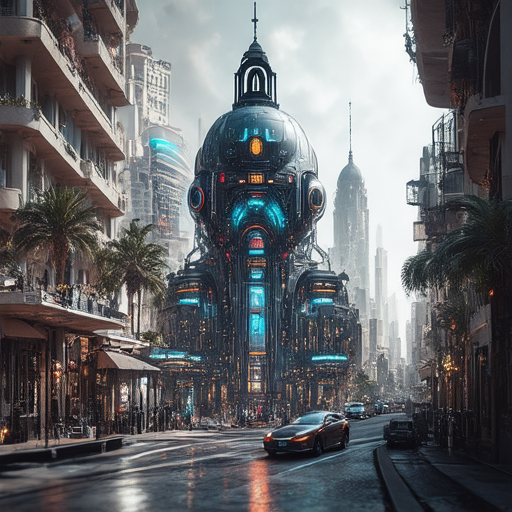}{cyberpunk Casablanca}\par
  \end{minipage}%
}

\newcommand{\preprinttitlefigure}{%
  \par\vspace{0.08cm}
  \noindent\begin{minipage}{\textwidth}
    \centering
    \preprinttitlefigurecontent\par
    \vspace{-0.12em}
    {\captionsetup{font={footnotesize,it,color=slate}}\captionof{figure}{Top: RAM reaches Flow-GRPO's peak training reward in up to $50\times$ fewer training steps (with even slightly lower per-step compute). Bottom: RAM improves SD3.5M generations on composability, text rendering, and human preference.}\label{fig:title_figure}}
  \end{minipage}\par
}

\makeatletter
\renewcommand{\preprint@titleblock}{%
  {\setlength{\parskip}{0pt}%
   \raggedright
   \nohyphens
   {\setstretch{1.3}\sffamily\fontsize{22}{25}\selectfont\bfseries\color{ink}\@title\par}%
   \ifx\@subtitle\@empty\else
     \vskip 0.18cm
     {\large\sffamily\color{slate}\@subtitle\par}%
   \fi
   \vskip 0.34cm
   {\authorlist\par}%
   \vskip 0.16cm
   {\affiliationlist\ifdefempty{\contributionlist}{}{\quad\contributionlist}\par}%
  }%
}

\renewcommand{\maketitle}{%
  \if@twocolumn
    \twocolumn[%
      \mymaketitle
      \preprinttitlefigure
      \vskip 0.38cm
    ]%
  \else
    \thispagestyle{plain}%
    \mymaketitle
    \preprinttitlefigure
    \vspace{1.0em}%
  \fi
}
\makeatother

\begin{document}

\maketitle
\clearpage

\section{Introduction}

Diffusion and flow-matching models dominate high-fidelity generation in continuous domains~\citep{sohl2015deep,ho2020ddpm,song2020scorebased,lipman2022flowmatching,liu2022rectifiedflow}. They scale because pretraining is supervised regression. A clean sample from the dataset is corrupted by Gaussian noise, and a model regresses against a closed-form target. That is the entire training procedure.

We often want to optimize generation towards a reward function, rather than just matching the data distribution. In image generation, this is what makes a model follow complex prompts~\citep{ghosh2023geneval}, render text legibly, or align with human preferences~\citep{xu2023imagereward}.
In language modeling, reinforcement learning (RL) post-training is what produces today's reasoning systems.

Given a pretrained model and a reward function, the canonical post-training objective is to maximize expected reward while staying close to the pretrained model in a KL sense. For models with tractable likelihoods, such as language models, policy gradients optimize this objective using the log-likelihood of sampled outputs. In diffusion and flow-matching models, sample likelihoods are intractable, so post-training is more challenging. 

Existing methods pay a different price to work around this. Policy-gradient methods make the denoising process stochastic and use Gaussian transition densities to optimize a variational lower bound~\citep{black2023training,fan2023dpok,liu2025flowgrpo}. Stochastic optimal control methods target the KL-regularized objective exactly, but current estimators require reward gradients backpropagated through the sampling process via an adjoint ODE~\citep{uehara2024fine,domingo2024adjoint}. Both approaches sample by SDE rollout, which requires many steps, and for flow-matching models the noise schedule diverges near the noisy end. To sidestep this, recent methods rely on surrogate objectives: \citet{xue2025awm} replace the intractable likelihood with a flow-matching ELBO, and \citet{zheng2025diffusionnft} distill an improvement direction by contrasting positive and negative endpoints.

We show that the regression structure of pretraining extends to RL post-training. Casting post-training as a stochastic optimal-control problem gives a fixed-point condition for the optimal control~\citep{domingo2024adjoint}. From this condition and a REINFORCE identity, we derive a reward-corrected regression target for the model. To scale this target, we use a key structural property of the optimum: it changes which clean samples are likely, but not how a fixed clean sample is noised. The rule that generated training pairs at pretraining is the same at the optimum. This lets us sample an endpoint from the current model using any off-the-shelf sampler, evaluate its reward, and noise it analytically as in pretraining. We reuse each endpoint for multiple noisy samples, amortizing the cost of sampling and reward evaluation. Our method, Reinforce Adjoint Matching (RAM), uses no SDE rollouts, no backward adjoint sweeps, and no reward gradients.

We post-train Stable Diffusion 3.5M on compositional generation (GenEval), visual text rendering (OCR), and human-preference alignment (PickScore).
RAM achieves the highest reward on each task without reward hacking or visible quality degradation. It matches Flow-GRPO's peak reward in $50\times$, $48\times$, and $34\times$ fewer training steps, at slightly lower per-step compute cost. The simple regression that scales pretraining now scales post-training too.

\section{Diffusion and Flow-Matching Models}
\label{sec:diffusion_and_flow}

Diffusion and flow-matching models share a common structure: a \emph{forward} process that corrupts data to noise from time $t{=}0$ to $t{=}1$, and a \emph{backward} process that reverses it. Diffusion models construct both processes explicitly. Flow-matching models specify only the noising kernel and generate by integrating the velocity field (the probability-flow ODE). The associated forward and backward SDEs are nonetheless well-defined.

\paragraph{Forward noising.}
We linearly interpolate between a data distribution~$p_0$ on~$\R^d$ at $t{=}0$ and the Gaussian prior $\mathcal N(0,\,I)$ at $t{=}1$~\citep{liu2022rectifiedflow,lipman2022flowmatching,esser2024scaling}. Non-linear interpolations recover other classical schedules such as DDPM~\citep{ho2020ddpm,sohl2015deep,song2020scorebased}, and all results in this paper generalize.
The interpolation defines a \emph{noising kernel} that sends a clean sample~$X_0\sim p_0$ to a noisy version at time~$t$,
\begin{equation}
  X_t \mid X_0
  \;\sim\;
  \mathcal N\!\bigl((1-t)\,X_0,\,t^2 I\bigr).
  \label{eq:noising_kernel}
\end{equation}
The kernel is realized by the linear forward SDE
\begin{equation}
  \mathd X_t
  =
  \kappa_t\,X_t\,\mathd t
  +
  \sigma_t\,\mathd B_t,
  \qquad
  X_0 \sim p_0,
  \label{eq:forward_process}
\end{equation}
with $\kappa_t = -1/(1-t)$, $\sigma_t^{2} = 2t/(1-t)$, and $B_t$ a standard Brownian motion on $\R^d$.
We write~$p_t$ for the marginal density of~$X_t$.

\paragraph{Velocity matching.}
The standard choice in large-scale generative modeling~\citep{liu2022rectifiedflow,esser2024scaling} is to learn a model~$v^\theta$ of the \emph{velocity field}
\begin{equation}
  v_t(x)
  =
  \E\!\left[\epsilon - X_0
             \;\middle|\; X_t=x\right],
  \qquad
  X_0 \sim p_0, \quad
  \epsilon \sim \mathcal N(0,\,I).
  \label{eq:velocity_field}
\end{equation}
By the tower property, regressing $v_t^\theta$ against the random target inside~\eqref{eq:velocity_field} recovers the conditional expectation. Because $X_t$ is an affine function of $X_0$ and $\epsilon$, each training pair is cheap to construct from a data sample and independent noise. The flow-matching loss is

\begin{graybox}
\begin{equation}
  \begin{gathered}
    \mathcal L_{\mathrm{FM}}(\theta)
    =
    \E_t\!\left[
      \bigl\|
        v_t^\theta(X_t)
        -
        (\epsilon - X_0)
      \bigr\|^2
    \right], \\
    \text{where }
    X_0 \sim p_0,\quad
    \epsilon \sim \mathcal N(0,\,I),\quad
    X_t = (1-t)\,X_0 + t \epsilon.
  \end{gathered}
  \label{eq:flow_matching_loss}
\end{equation}
\end{graybox}

Alternative parametrizations that predict the clean data or the noise instead are equivalent: given $X_t$, any two of the velocity, clean data, and noise determine the third.

\paragraph{Backward generation.}
The time reversal of the forward SDE~\eqref{eq:forward_process}~\citep{anderson1982reverse} is
\begin{equation}
  \mathd X_t
  =
  \bigl(
    \kappa_t X_t - \sigma_t^{2} \nabla_x\!\log p_t(X_t)
  \bigr)\mathd t
  +
  \sigma_t\,\mathd B_t,
  \qquad
  X_1 \sim \mathcal N(0,\,I),
  \label{eq:backward_process}
\end{equation}
which we integrate from~$t{=}1$ to~$t{=}0$ to sample from~$p_0$.
The score is related to the velocity by
\begin{equation}
  \nabla_x\!\log p_t(x)
  =
  \frac{2}{\sigma_t^{2}}
  \bigl(\kappa_t x-v_t(x)\bigr),
  \label{eq:score_from_velocity}
\end{equation}
so a learned velocity model~$v^\theta$ suffices to sample.
More generally, replacing the noise coefficient in~\eqref{eq:backward_process} by any other schedule $\tilde\sigma_t \ge 0$ (with correspondingly adjusted drift) preserves the marginals~$p_t$~\citep{albergo2023stochastic,song2020scorebased}.
The null choice $\tilde\sigma{\equiv}0$ recovers the probability-flow ODE $\mathd X_t = v_t(X_t)\,\mathd t$, which is the default sampler for flow-matching models.

\section{RL Post-Training as Optimal Control}
\label{sec:finetuning_optimal_control}

Given a pretrained generative model with endpoint distribution $p_0^{\mathrm{ref}}$ and a scalar reward function $r:\R^d\to\R$, the canonical KL-regularized target is the \emph{tilted distribution}
\begin{equation}
  p_0^\ast
  =
  \argmax_p
  \bigl\{
    \E_{x\sim p}[r(x)]
    -
    \KL\!\bigl(p \| p_0^{\mathrm{ref}}\bigr)
  \bigr\}
  \;\propto\;
  p_0^{\mathrm{ref}}(x)\exp\bigl(r(x)\bigr).
  \label{eq:tilted_distribution}
\end{equation}
The KL term keeps the post-trained model close to the pretrained reference and within its support. The reward shifts mass toward desirable samples.
The regularization strength is controlled by scaling the reward, which we leave implicit for notational simplicity.

\paragraph{Stochastic optimal control.}
Sampling directly from~$p_0^\ast$ is intractable.
We instead steer the pretrained backward SDE toward it with a drift correction~$u_t$,
\begin{equation}
  \mathd X_t^u
  =
  \bigl(
    b_t^{\mathrm{ref}}(X_t^u)
    +
    \sigma_t u_t(X_t^u)
  \bigr)\mathd t
  +
  \sigma_t\,\mathd B_t,
  \qquad
  X_1^u \sim \mathcal N(0,\,I),
  \label{eq:controlled_process}
\end{equation}
where $b_t^{\mathrm{ref}}(x) = \kappa_t x - \sigma_t^{2} \nabla_x\!\log p_t^{\mathrm{ref}}(x)$ is the drift of the reference backward process~\eqref{eq:backward_process}.
We choose~$u$ to maximize the terminal reward minus a quadratic control cost,
\begin{equation}
  \max_u\;
  \E\!\left[
    r(X_0^u)
    -
    \frac12 \int_0^1 \|u_\tau(X_\tau^u)\|^2\,\mathd\tau
  \right].
  \label{eq:SOC}
\end{equation}
By Girsanov's theorem, this stochastic optimal control problem is equivalent to KL-regularized reward maximization in path space (\Cref{app:ocp_as_kl}).

\paragraph{Structure of the optimal process.}
The next theorem is a standard consequence of stochastic optimal control theory~\citep{pham2009continuous}: optimality is equivalent to tilting the clean-endpoint distribution to $p_0^\ast$ while leaving the conditional law of noisy states given that clean endpoint unchanged.
This recovers the memoryless-schedule result of~\citet{domingo2024adjoint} as the marginal-time statement of a full path-space characterization.

\begin{theorem}[Optimal controlled process]
  \label{thm:optimal_controlled_process}
  Under standard regularity assumptions (\Cref{app:optimal_controlled_process}), a controlled process~$X^u$ solves~\eqref{eq:SOC} if and only if
  \begin{equation}
    X_0^u \sim p_0^\ast
    \quad\text{and}\quad
    \Law\!\left(
      (X_t^u)_{t\in[0,1]}
      \;\middle|\;
      X_0^u = x_0
    \right)
    =
    \Law\!\left(
      (X_t)_{t\in[0,1]}
      \;\middle|\;
      X_0 = x_0
    \right)
    \label{eq:optimal_process}
  \end{equation}
  for $p_0^\ast$-almost every~$x_0$.
  Equivalently, the optimal controlled process is the time reversal of the forward noising~\eqref{eq:forward_process} started from~$p_0^\ast$.
\end{theorem}

The theorem is the precise sense in which RL post-training changes \emph{which} clean samples are preferred, but not the noising rule that connects a fixed clean sample to its noisy versions.
The reward tilt acts only on~$X_0$, and the conditional trajectory law given~$X_0$ depends only on the interpolation schedule, not on the data distribution.

Two consequences follow.
First, the conditional law of a noisy state given the clean endpoint is the same analytic Gaussian kernel as in pretraining~\eqref{eq:noising_kernel}. Only the endpoint distribution changes, from $p_0$ to $p_0^\ast$.
Second, the optimal marginal $p_t^\ast$, velocity $v_t^\ast$, and score $\nabla_x\!\log p_t^\ast$ take the same functional forms as~\eqref{eq:velocity_field} and~\eqref{eq:score_from_velocity}, with $p_0^\ast$ in place of $p_0$.

\paragraph{Adjoint matching.}
To reach this fixed point in practice, we introduce the Bellman value function
\begin{equation}
  V_t^u(x)
  =
  \E\!\left[
    r(X_0^u)
    -
    \frac12 \int_0^t \|u_\tau(X_\tau^u)\|^2\,\mathd\tau
    \;\middle|\;
    X_t^u = x
  \right],
  \label{eq:value_function}
\end{equation}
and its spatial gradient $A_t^u(x) \coloneqq \nabla_x V_t^u(x)$, which we call the \emph{adjoint}.
The following verification result turns optimal control into a self-consistency condition on~$u$.

\begin{theorem}[Adjoint matching]
  \label{thm:adjoint_matching}
  Under standard regularity assumptions (\Cref{app:verification_fixed_points}), a control~$u$ is optimal if and only if
  \begin{equation}
    u_t(x)
    =
    -\sigma_t A_t^u(x)
    \qquad
    \text{for all }(t,x)\in[0,1]\times\R^d.
    \label{eq:value_function_matching}
  \end{equation}
\end{theorem}
Following \citet{domingo2024adjoint}, we turn this fixed-point condition into on-policy training: estimate the adjoint~$A_t^u$ under the current control, and regress~$u_t$ against the resulting target.
In practice we parametrize the velocity field rather than~$u$ directly, so we derive a regression target for~$v_t^\theta$ via the control--velocity relation~\eqref{eq:control_to_velocity}.

\section{Reinforce Adjoint Matching}
\label{sec:ram}

\begin{figure}[t]
  \centering
  \includegraphics[width=\linewidth]{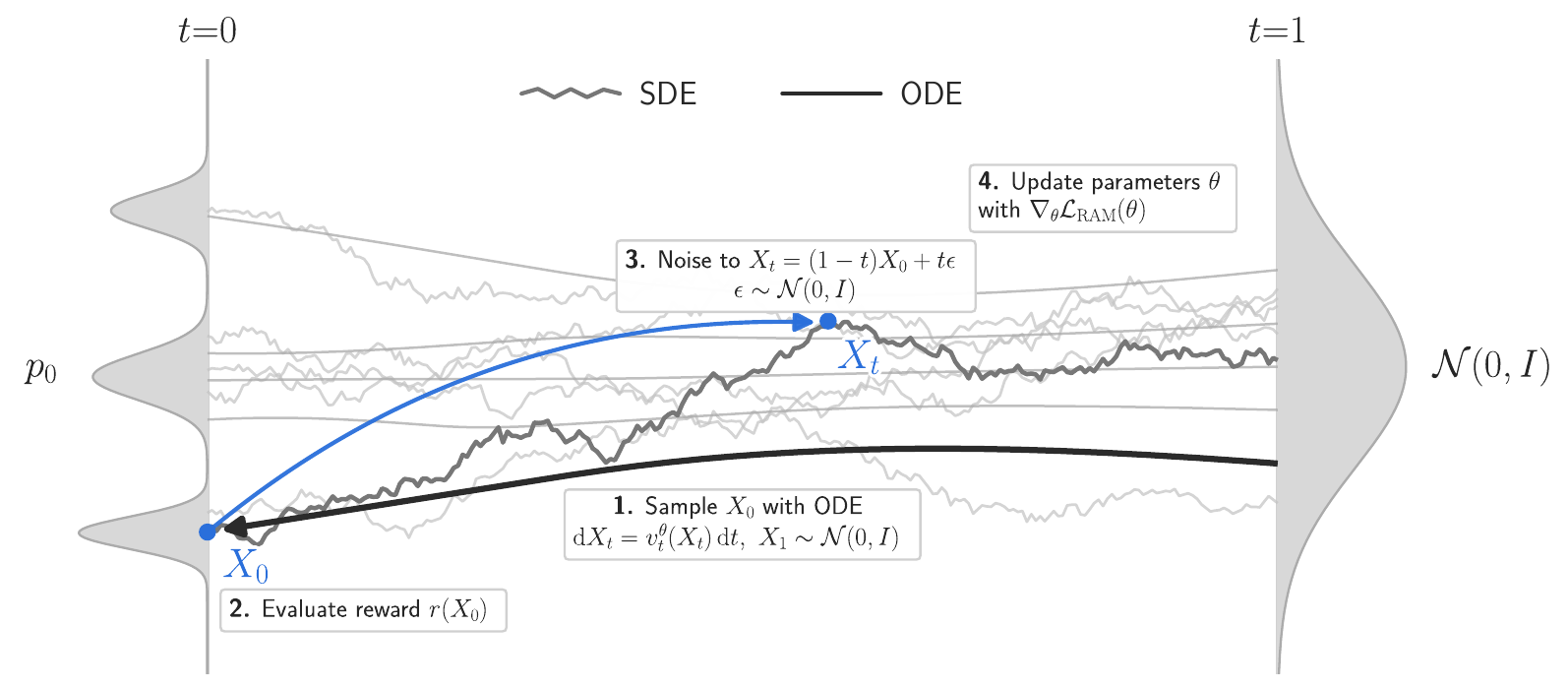}
  \caption{
    \textbf{RAM training.}
    We draw a clean endpoint with an ODE sampler, evaluate its reward, and noise it as in pretraining to obtain training states.
    This avoids SDE rollouts, and reusing each endpoint across many noise draws amortizes the cost of sampling and reward evaluation.
  }
  \label{fig:ram_overview}
\end{figure}

\emph{Reinforce Adjoint Matching} (RAM) is an on-policy consistency loss for the velocity field that requires no reward gradients.
By linearity of expectation, the value function~\eqref{eq:value_function} splits into an endpoint-reward term and an integrated running-cost term.
Differentiating the endpoint term in~$x$ naively produces a reward gradient~$\nabla r$, which we avoid since~$r$ may be non-differentiable.
Applying the log-derivative (REINFORCE) identity replaces this gradient with the backward bridge score weighted by the reward, yielding the exact decomposition
\begin{equation}
    A_t^u(x)
    =
    \underbrace{
      \E\!\left[
        r(X_0^u) \nabla_x\!\log p_{0|t}^u(X_0^u \mid x)
        \;\middle|\;
        X_t^u = x
      \right]
    }_{\text{reward term}}
    -
    \frac12
    \underbrace{
      \E\!\left[
        \nabla_x\!\int_0^t \|u_\tau(X_\tau^u)\|^2\,\mathd\tau
        \;\middle|\;
        X_t^u = x
      \right]
    }_{\text{path-cost correction}}.
  \label{eq:exact_adjoint_decomposition}
\end{equation}
The reward term admits a closed-form Monte Carlo estimator, which we develop below.
For the path-cost correction, \Cref{sec:path_cost} presents several estimators trading off exactness, variance, and computational cost.
At image scale, either variance is too high or the estimator requires a backward adjoint sweep along a stored SDE rollout, bringing numerical instability and substantial per-step compute. We therefore use the approximation
\begin{equation}
  A_t^u(x)
  \approx
  \E\!\left[
    r(X_0^u) \nabla_x\!\log p_{0|t}^u(X_0^u \mid x)
    \;\middle|\;
    X_t^u = x
  \right].
  \label{eq:ram_reward_proxy_identity}
\end{equation}
This approximation coincides with the full adjoint at initialization, where $u{\equiv}0$ makes the path cost vanish, and exactly for a Gaussian reference under a linear reward (\Cref{app:gaussian_case}). \Cref{sec:ram_fp_analysis} relates the resulting RAM fixed point to the KL-regularized optimum more generally.
The pretrained reference remains an explicit anchor in the regression target~\eqref{eq:ram_velocity_loss}, penalizing deviation from $v^{\mathrm{ref}}$ throughout training.
In practice, this prevents reward hacking despite RAM achieving the highest task reward among all baselines (\Cref{sec:experiments}).

\paragraph{Endpoint sampling and analytic noising.}
Equation~\eqref{eq:ram_reward_proxy_identity} requires joint samples $(X_0^u, X_t^u)$. In principle these come from simulating the controlled SDE~\eqref{eq:controlled_process}, which requires many steps and is numerically unstable for flow-matching models. \Cref{thm:optimal_controlled_process} opens up a much cheaper route: at the optimum, the bridge from clean endpoint to noisy state is the pretraining kernel~\eqref{eq:noising_kernel}. We therefore draw an on-policy endpoint $X_0$ with any off-the-shelf ODE sampler and noise it: $X_t = (1-t) X_0 + t \epsilon$.

This recovers the controlled-SDE joint $(X_0^u, X_t^u)$ exactly at initialization and at the optimum. Since the training objective is a fixed-point condition, exactness is formally needed only at the optimum. In practice, endpoint sampling works well throughout training. Each endpoint then yields $K$ independent training states at the cost of one model sample and one reward query (\Cref{fig:ram_overview} and~\Cref{alg:ram}). Where existing SDE-based methods produce correlated states from a single rollout, RAM's $K$ states are conditionally independent given the endpoint. This translates into more gradient signal per step (\Cref{sec:experiments}).

\paragraph{Bayes bridge score.}
The reward term in~\eqref{eq:ram_reward_proxy_identity} requires the backward bridge score.
By Bayes' rule, and noting that $p_0(x_0)$ is constant in~$x_t$, $\nabla_{x_t}\!\log p_{0|t}(x_0 \mid x_t) = \nabla_{x_t}\!\log p_{t|0}(x_t \mid x_0) - \nabla_{x_t}\!\log p_t(x_t)$.
The forward-bridge score is analytic from~\eqref{eq:noising_kernel} and the marginal score follows from~\eqref{eq:score_from_velocity}, giving a closed-form expression in $v_t$.

\begin{proposition}[Bayes bridge score]
  \label{prop:ram_bridge_score}
  For the noising kernel~\eqref{eq:noising_kernel}, velocity field~\eqref{eq:velocity_field}, $0 < t \le 1$, and $\epsilon \coloneqq (x_t - (1-t) x_0)/t$,
  \begin{equation}
    \nabla_{x_t}\!\log p_{0|t}(x_0 \mid x_t)
    =
    \frac{1-t}{t}
    \bigl(
      v_t(x_t)
      -
      (\epsilon - x_0)
    \bigr).
    \label{eq:general_bayes_bridge_score}
  \end{equation}
\end{proposition}

The full algebra is given in~\Cref{app:bayes_bridge_score}.
The identity is exact for the uncontrolled reference process at initialization, and by \Cref{thm:optimal_controlled_process} it also holds for the optimal process.
Away from those points we substitute $v_t^\theta$ for $v_t$ as a plug-in approximation.

\paragraph{RAM objective.}
We train the velocity field $v^\theta$ rather than parametrizing $u_t$ directly, initializing $v^\theta = v^{\mathrm{ref}}$.
The two are related by
\begin{equation}
  \sigma_t u_t
  =
  2\bigl(v_t^\theta - v_t^{\mathrm{ref}}\bigr),
  \label{eq:control_to_velocity}
\end{equation}
matching the true control--velocity relation at the optimum (\Cref{app:proof_optimal_score_corollary}).
Substituting~\eqref{eq:ram_reward_proxy_identity} and~\eqref{eq:general_bayes_bridge_score} into the adjoint-matching condition~\eqref{eq:value_function_matching} gives a Monte Carlo target for the velocity field.
The time-dependent prefactors cancel (algebra in~\Cref{app:ram_control_space}), yielding the RAM loss

\begin{graybox}
\begin{equation}
  \begin{gathered}
    \mathcal L_{\mathrm{RAM}}(\theta)
    =
    \E_t\!\left[
      \left\|
        v_t^\theta(X_t)
        -
        \mathrm{sg}\!\left(
          v_t^{\mathrm{ref}}(X_t)
          +
          r(X_0)\bigl(
            (\epsilon - X_0)
            -
            v_t^\theta(X_t)
          \bigr)
        \right)
      \right\|^2
    \right], \\
    \text{where }
    X_0 \sim p_0^\theta,\quad
    \epsilon \sim \mathcal N(0,\,I),\quad
    X_t = (1-t) X_0 + t \epsilon.
  \end{gathered}
  \label{eq:ram_velocity_loss}
\end{equation}
\end{graybox}

Here $\mathrm{sg}(\cdot)$ is the stop-gradient operator, and we also treat the endpoint sampling as constant, so gradients are taken only with respect to the leading $v_t^\theta(X_t)$ term. A natural alternative is to form the regression loss in control space and convert via~\eqref{eq:control_to_velocity}. This introduces a factor $4/\sigma_t^{2}$ that downweights the early timesteps that shape the trajectory. Avoiding this scaling was necessary for stable training (\Cref{app:ram_control_space}).

\begin{algorithm}[t]
  \caption{RAM}
  \label{alg:ram}
  \begin{algorithmic}[1]
    \Require parameters $\theta$ initialized from the pretrained reference field $v^{\mathrm{ref}}$, reward $r$, targets per endpoint $K$
    \While{not converged}
      \State sample $x_0$ from $p_0^\theta$ \Comment{any sampler}
      \State $r_0 \gets r(x_0)$
      \For{$k = 1, \ldots, K$} \textbf{in parallel}
        \State sample time $t_k$ and noise $\epsilon_k \sim \mathcal N(0,\,I)$
        \State $x_k \gets (1-t_k) x_0 + t_k \epsilon_k$
        \State $\hat v_k \gets v_{t_k}^{\mathrm{ref}}(x_k) + r_0\bigl((\epsilon_k - x_0)-v_{t_k}^\theta(x_k)\bigr)$
      \EndFor
      \State update $\theta$ with $\nabla_\theta \frac{1}{K}\sum_{k=1}^K\|v_{t_k}^\theta(x_k)-\mathrm{sg}\!\left(\hat v_k\right)\|^2$ \Comment{$\mathrm{sg}$ stops gradients}
    \EndWhile
  \end{algorithmic}
\end{algorithm}

\subsection{Relating the RAM fixed point to the KL optimum}
\label{sec:ram_fp_analysis}

The training objective~\eqref{eq:ram_velocity_loss} implies the fixed-point equation
\begin{equation}
  v_t^\theta(x) - v_t^{\mathrm{ref}}(x)
  =
  \E\!\left[
    r(X_0)\bigl((\epsilon - X_0) - v_t^\theta(x)\bigr)
    \;\middle|\;
    X_t = x
  \right],
  \quad X_0 \sim p_0^\theta,
  \label{eq:ram_fp}
\end{equation}
where $p_0^\theta$ is the endpoint distribution generated by integrating the ODE under $v_t^\theta$.
To relate this self-consistency condition to the KL-regularized optimum, we connect the reference and the optimum through a path of exponentially tilted endpoints.

\begin{lemma}[Path-integral characterization of the optimum]
\label{lem:tilt_path}
For $\lambda \in [0,\,1]$, let $p_0^\lambda(x_0) \propto p_0^{\mathrm{ref}}(x_0)\, e^{\lambda r(x_0)}$ interpolate between $p_0^0 = p_0^{\mathrm{ref}}$ and $p_0^1 = p_0^\ast$, and let $v_t^\lambda$ denote the velocity field for endpoint distribution $p_0^\lambda$. Then
\begin{equation}
  \partial_\lambda v_t^\lambda(x)
  =
  \E\!\left[
    r(X_0)\bigl((\epsilon - X_0) - v_t^\lambda(x)\bigr)
    \;\middle|\;
    X_t = x
  \right],
  \quad X_0 \sim p_0^\lambda,
  \label{eq:tilt_velocity_ode}
\end{equation}
and integrating over $\lambda \in [0,\,1]$ yields
\begin{equation}
  v_t^\ast(x) - v_t^{\mathrm{ref}}(x)
  =
  \int_0^1
  \E\!\left[
    r(X_0)\bigl((\epsilon - X_0) - v_t^\lambda(x)\bigr)
    \;\middle|\;
    X_t = x
  \right]
  \mathd\lambda,
  \quad X_0 \sim p_0^\lambda.
  \label{eq:exact_tilt_integral}
\end{equation}
\end{lemma}

The proof is given in~\Cref{app:proof_tilt_path}.
Equation~\eqref{eq:ram_fp} has the same form as the integrand of~\eqref{eq:exact_tilt_integral} evaluated at the right endpoint $\lambda{=}1$, with the on-policy distribution $p_0^\theta$ in place of $p_0^\ast$ and $v_t^\theta$ in place of $v_t^\ast$. The integrand itself is the conditional reward-velocity covariance $\mathrm{Cov}\bigl(r(X_0),\,\epsilon - X_0 \;\big|\; X_t = x\bigr)$ with $X_0 \sim p_0^\lambda$. RAM thus replaces the average of this covariance over the path by a single evaluation at $\lambda{=}1$. The rule is exact when $v_t^\lambda$ depends linearly on $\lambda$ along the path, in particular for a Gaussian reference under a linear reward (\Cref{app:gaussian_case}). More generally, RAM and the KL optimum agree to first order in the reward: scaling $r \mapsto \eta r$ and expanding around $\eta{=}0$, both~\eqref{eq:ram_fp} and~\eqref{eq:exact_tilt_integral} reduce at leading order to $\eta\,\mathrm{Cov}\bigl(r(X_0),\,\epsilon - X_0 \;\big|\; X_t = x\bigr)$ with $X_0 \sim p_0^{\mathrm{ref}}$.

\section{Retaining the Path-Cost Correction}
\label{sec:path_cost}

RAM drops the path-cost correction in the adjoint decomposition~\eqref{eq:exact_adjoint_decomposition}. 
To show that this is the right choice for scaling to image generation, we examine the challenges of retaining it.
First, we can preserve RAM's endpoint sampling and analytic noising and estimate the path cost at one random intermediate time. The price is high variance. Alternatively, we can simulate the controlled SDE~\eqref{eq:controlled_process}, store the trajectory, and differentiate the running cost pathwise. The length of pathwise differentiation trades variance against compute. \Cref{tab:estimators} summarizes the landscape.

\begin{table}[h]
  \centering \small
  \begin{tabular}{@{}llll@{}}
    \toprule Estimator & Samples from & Exactness & Trade-off \\ \midrule RAM (\Cref{sec:ram}) & endpoint + analytic noising & biased & drops path cost \\ Random jump & endpoint + analytic noising & exact at init./optimum & high variance \\ Full-horizon w/ Bayes & SDE rollout & exact at init./optimum & pathwise VJPs \\ Full-horizon w/ Malliavin & SDE rollout & exact & pathwise + score VJPs \\ Local & SDE rollout & exact & high variance \\ \bottomrule
  \end{tabular}
  \caption{Estimators for the value-function gradient~$A_t^u$.} \label{tab:estimators}
\end{table}

\paragraph{Endpoint sampling with a random jump.}
We insert a uniformly random intermediate time $s \in [0, t)$ into the analytic noising trajectory. Then $\tfrac{t}{2}\|u_s(X_s)\|^2$ is an unbiased estimate of the integrated path cost. Combined with the Bayes bridge score from intermediate to training state, this yields the jump estimator
\begin{equation}
  \widehat A_t^{\mathrm{jump}}
  =
  \bigl(
    r(X_0) - \tfrac{t}{2} \|u_s(X_s)\|^2
  \bigr)
  \nabla_{x_t}\!\log p_{s|t}(X_s \mid x_t)
  \big|_{x_t = X_t},
  \quad
  s \sim \mathcal U[0,t).
  \label{eq:jump_estimator}
\end{equation}
The estimator is cheap and exact at initialization and at the optimum (\Cref{app:endpoint_jump_derivation}). The price is variance: the entire gradient signal flows through that one scalar prefactor.

\paragraph{Pathwise differentiation along an SDE rollout.}
We simulate the controlled SDE~\eqref{eq:controlled_process} and recover the path-cost gradient by integrating an adjoint ODE backward along the stored trajectory. This is functionally equivalent to differentiating through the SDE solver in autograd, but more memory efficient.

Adjoint integration over the full horizon is expensive and numerically delicate. Dynamic programming lets us replace some of the pathwise integration with a REINFORCE term: at any intermediate time $s$, the adjoint splits into a REINFORCE term over the prefix $[0, s]$ and a pathwise gradient over the suffix $[s, t]$.

\begin{theorem}[Generalized adjoint]
  \label{thm:generalized_ram}
  For every $0 \le s < t \le 1$ and admissible control $u$,
  \begin{equation}
    \nabla_x V_t^u(x)
    =
    \underbrace{
      \E\!\left[
        V_s^u(X_s^u)\,\nabla_x\!\log p_{s|t}^u(X_s^u \mid x)
        \;\middle|\; X_t^u = x
      \right]
    }_{\text{REINFORCE on prefix}}
     -
    \frac{1}{2}
    \underbrace{
      \E\!\left[
        \nabla_x\!\int_s^t \|u_\tau(X_\tau^u)\|^2\,\mathd\tau
        \;\middle|\; X_t^u = x
      \right]
    }_{\text{pathwise suffix}}.
    \label{eq:generalized_adjoint}
  \end{equation}
\end{theorem}
The proof differentiates the recursion in $x$ and applies the log-derivative identity to the prefix term (\Cref{app:proof_generalized_ram}).

A single-trajectory estimator of the right-hand side combines three pieces: a running prefix value (the reward minus the integrated path cost over $[0, s]$), a bridge score, and the suffix path-cost gradient. For the bridge score we can use the Bayes estimator~\eqref{eq:general_bayes_bridge_score}, which generalizes in closed form to all $s \in [0, t)$ (\Cref{app:bayes_bridge_score}). Alternatively, the Malliavin score of~\citet{pidstrigach2025conditioning} is exact even away from the optimum, at the cost of additional VJPs through the score.

Two values of $s$ are computationally practical. \emph{Full-horizon} ($s = 0$) computes targets for every $t$ in a single backward sweep. \emph{Local} ($s = t - \delta$) avoids the sweep, collapsing the suffix to a single discretization step. Intermediate values would require a separate sweep per target time.

\subsection{Estimator variance on a 2D example}
\label{sec:variance_experiments}

\begin{figure}[h]
  \centering
  \includegraphics[width=0.8\linewidth]{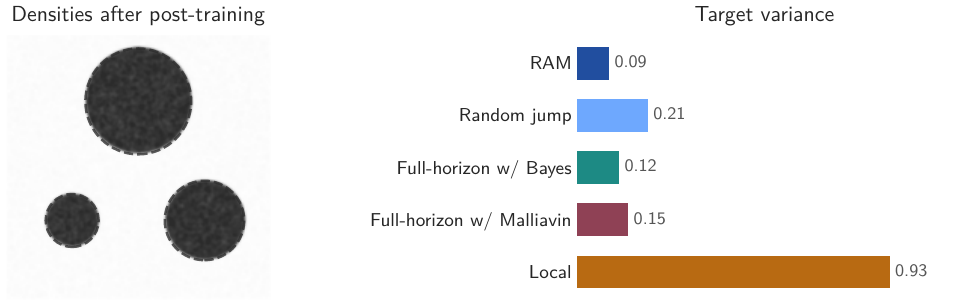}
  \caption{
    \textbf{Path-cost-corrected estimators on a 2D toy.}
    Left: post-trained densities match the tilted target for every estimator. Right: regression loss at convergence, equal to the variance of the control-space target.
  }
  \label{fig:toy_experiments}
\end{figure}

We compare the estimators on a 2D rectified flow with uniform endpoint on $[-1, 1]^2$ under a reward that is positive inside three circles of varying radius. All recover the correct tilted distribution, confirming the shared fixed point. At convergence the regression loss collapses to the variance of the control-space target $\sigma_t(\widehat A_{s,t}^u - A_t^u)$, so the right panel reads off estimator variance directly. Among the path-cost-corrected estimators, full-horizon has substantially lower variance than local or random-jump, and the Bayes bridge score beats Malliavin within full-horizon. RAM attains the lowest variance overall: dropping the path-cost gradient eliminates a dominant source of noise. These trade-offs amplify at image scale: variance grows with dimension, and adjoint sweeps over long horizons become prohibitive. RAM is the only estimator that scales to text-to-image post-training.

\section{Text-to-Image Experiments}
\label{sec:experiments}

We post-train Stable Diffusion 3.5 Medium (SD3.5M)~\citep{esser2024scaling}, a 2.5B-parameter rectified-flow transformer, on three text-to-image reward objectives, and compare RAM against the strongest recent baselines: Flow-GRPO~\citep{liu2025flowgrpo}, DiffusionNFT~\citep{zheng2025diffusionnft}, and AWM~\citep{xue2025awm}.
We largely follow the evaluation setup of~\citet{liu2025flowgrpo}: for each reward we train a separate model on its benchmark prompts, report the training reward on held-out prompts from the same benchmark, and check for reward hacking on an independent prompt set.
Flow-GRPO numbers come from the original paper, except for HPSv2, which we recompute from the authors' released checkpoints.
We retrain DiffusionNFT and AWM.
We train each model until the reward plateaus or starts to decrease.
Qualitative samples from the pretrained model and from the models post-trained with RAM are shown in \Cref{fig:title_figure}.

\paragraph{Training rewards.}
Each benchmark comes with its own training and test prompt sets.
During training we sample images for the training prompts and score each prompt-image pair under the reward function associated with that benchmark.
GenEval~\citep{ghosh2023geneval} measures compositional correctness: for prompts such as \emph{``a truck to the left of a refrigerator''}, pretrained vision models verify whether the required objects, attributes, and spatial relations all appear in the generated image.
OCR evaluates visual text rendering via an edit-distance reward that checks whether text specified in the prompt appears legibly in the image, following~\citet{liu2025flowgrpo}.
PickScore~\citep{kirstain2023pickapic} is a learned human-preference model trained on large-scale pairwise image comparisons.

\paragraph{Image-quality evaluation.}
Optimizing a single reward can degrade generic image quality, a failure mode known as reward hacking.
To detect this, we score each post-trained model on DrawBench~\citep{saharia2022photorealistic} prompts disjoint from training and test under five off-the-shelf metrics: Aesthetic~\citep{schuhmann2022laion} and DeQA~\citep{you2025deqa} rate perceptual quality; ImageReward~\citep{xu2023imagereward}, HPSv2~\citep{wu2023human}, and PickScore~\citep{kirstain2023pickapic} are learned human-preference models.

\begin{figure}[t]
  \centering
  \begin{subfigure}[t]{0.32\textwidth}
    \centering
    {\small\sffamily\bfseries Compositional Image Generation\par}
    \vspace{0.2em}
    \includegraphics[width=\linewidth]{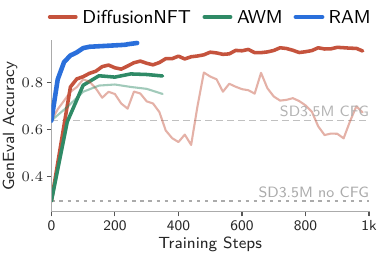}
  \end{subfigure}
  \hfill
  \begin{subfigure}[t]{0.32\textwidth}
    \centering
    {\small\sffamily\bfseries Visual Text Rendering\par}
    \vspace{0.2em}
    \includegraphics[width=\linewidth]{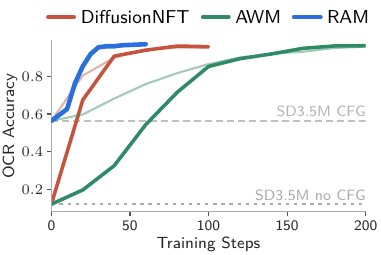}
  \end{subfigure}
  \hfill
  \begin{subfigure}[t]{0.32\textwidth}
    \centering
    {\small\sffamily\bfseries Human Preference Alignment\par}
    \vspace{0.2em}
    \includegraphics[width=\linewidth]{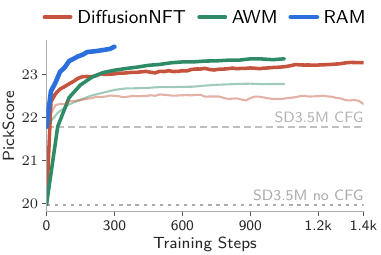}
  \end{subfigure}
  \caption{
    Training-reward curves comparing RAM with DiffusionNFT and AWM, two methods designed for training efficiency.
    The lighter companion curves show DiffusionNFT and AWM re-evaluated with CFG, which lowers their task reward.
    RAM, by contrast, remains compatible with CFG after post-training and is plotted with CFG.
  }
  \label{fig:training_speed_efficiency_methods}
\end{figure}

\paragraph{Reward normalization.}
Following common practice, we normalize raw rewards using group-relative advantage estimation.
For each training prompt we sample a group of $G = 24$ images, evaluate their rewards, and subtract the group mean.
We divide by the standard deviation pooled over all samples in the current training step, rather than per group. This stabilizes the scale across training and avoids the instability of per-group normalization when a group happens to draw nearly identical rewards.

\begin{table}[t]
  \centering
  \caption{
    Stable Diffusion 3.5 Medium (SD3.5M) post-training results.
    We report training rewards on held-out test prompts and image-quality metrics on DrawBench prompts.
    Higher is better for all metrics.
    Within each task, the best results are shown in \textbf{bold} and second-best results are \underline{underlined}.
    $\dagger$ denotes results obtained by retraining the corresponding baseline.
  }
  \label{tab:sd35_results}
  \small
  \setlength{\tabcolsep}{4pt}
  \renewcommand{\arraystretch}{1.08}
  \begin{tabular}{@{}lccccccccc@{}}
    \toprule
    \textbf{Model} & \textbf{\# Steps} & \multicolumn{3}{c}{\textbf{Training Reward}} & \multicolumn{5}{c}{\textbf{Image Quality Metrics}} \\
    \cmidrule(lr){3-5}
    \cmidrule(lr){6-10}
    & & \textbf{GenEval} & \textbf{OCR} & \textbf{PickScore} & \textbf{Aesthetic} & \textbf{DeQA} & \textbf{ImgRwd} & \textbf{HPSv2} & \textbf{PickScore} \\
    \midrule
    SD3.5M & & 0.64 & 0.56 & 21.79 & 5.41 & 4.08 & 0.82 & 0.28 & 22.40 \\
    \midrule
    \multicolumn{10}{c}{\textit{Compositional Image Generation}} \\
    \midrule
    Flow-GRPO & $>{}5$k & \underline{0.95} & & & \underline{5.25} & 4.01 & \underline{1.03} & \underline{0.27} & \underline{22.37} \\
    AWM$^\dagger$ & 300 & 0.83 & & & 5.14 & 3.75 & 0.67 & 0.24 & 22.04 \\
    DiffusionNFT$^\dagger$ & 900 & \underline{0.95} & & & 4.98 & \textbf{4.10} & 0.30 & 0.24 & 21.59 \\
    RAM & 270 & \textbf{0.97} & & & \textbf{5.38} & \underline{4.09} & \textbf{1.19} & \textbf{0.29} & \textbf{22.52} \\
    \midrule
    \multicolumn{10}{c}{\textit{Visual Text Rendering}} \\
    \midrule
    Flow-GRPO & $1.2$k & & 0.92 & & \textbf{5.32} & \textbf{4.06} & \textbf{0.95} & \textbf{0.28} & \textbf{22.44} \\
    AWM$^\dagger$ & 200 & & \textbf{0.97} & & 5.01 & 2.83 & -0.85 & 0.18 & 20.56 \\
    DiffusionNFT$^\dagger$ & 100 & & \underline{0.96} & & 4.87 & 3.01 & -0.97 & 0.18 & 20.26 \\
    RAM & 60 & & \textbf{0.97} & & \underline{5.23} & \underline{3.90} & \underline{0.44} & \underline{0.26} & \underline{21.83} \\
    \midrule
    \multicolumn{10}{c}{\textit{Human Preference Alignment}} \\
    \midrule
    Flow-GRPO & $>{}3$k & & & 23.31 & 5.92 & \textbf{4.22} & \underline{1.28} & \textbf{0.32} & 23.53 \\
    AWM$^\dagger$ & $1$k & & & \underline{23.39} & \textbf{6.31} & 4.10 & 1.27 & \underline{0.31} & \underline{23.76} \\
    DiffusionNFT$^\dagger$ & $1.4$k & & & 23.29 & \underline{6.16} & 4.13 & 1.23 & \underline{0.31} & 23.65 \\
    RAM & 300 & & & \textbf{23.67} & 6.11 & \underline{4.17} & \textbf{1.36} & \textbf{0.32} & \textbf{23.95} \\
    \bottomrule
  \end{tabular}
\end{table}

\paragraph{Results.}
RAM achieves the highest training reward on all three tasks (\Cref{tab:sd35_results}).
For compositional generation and preference alignment, this improvement comes with comparable or better image-quality metrics.
For visual text rendering, RAM avoids the severe quality collapse of DiffusionNFT and AWM while matching or exceeding their OCR reward.
Flow-GRPO scores higher than RAM on image-quality metrics for this task, but at a substantially lower OCR reward.
The OCR reward signal pushes models to sacrifice aesthetic quality for legibility. Stopping RAM earlier or using a smaller reward coefficient yields higher image-quality numbers at a lower OCR reward.
Qualitative samples in \Cref{fig:title_figure} show that RAM corrects compositional errors (\emph{GenEval}), renders specified text legibly (\emph{OCR}), and produces more preferred images (\emph{PickScore}) while preserving the style and fidelity of the pretrained model.
Notably, RAM remains compatible with classifier-free guidance (CFG) after post-training, whereas CFG lowers the task reward of DiffusionNFT and AWM (lighter curves in \Cref{fig:training_speed_efficiency_methods}).
We therefore evaluate RAM with CFG and the two baselines without.

\paragraph{Training efficiency.}
\Cref{fig:title_figure} (top) plots training reward against training steps for RAM and Flow-GRPO on all three tasks.
RAM matches Flow-GRPO's peak reward in roughly $50\times$ fewer steps on GenEval ($0.95$), $48\times$ fewer on OCR ($0.92$), and $34\times$ fewer on PickScore ($23.31$).
On GenEval, RAM reaches a reward of $0.90$ in $75\times$ fewer steps than Flow-GRPO.
Per-step compute cost is comparable between the two (e.g.\ $0.66$ GPU-hours for RAM vs.\ $0.70$ for Flow-GRPO on GenEval), so step ratios translate directly to wall-clock training time (\Cref{sec:experimental_details}).
All methods use the same number of prompts and samples per step, so differences in step efficiency reflect how each sample is used for learning.
Flow-GRPO performs one SDE rollout per sample and forms its loss from the strongly correlated noisy states along that trajectory.
RAM instead draws a clean endpoint and noises it $K = 8$ times with independent noise vectors. The resulting training states are conditionally independent given the endpoint, providing more gradient signal per step.
\Cref{fig:training_speed_efficiency_methods} extends the comparison to DiffusionNFT and AWM, two methods also designed for training efficiency.
These baselines amortize endpoints across multiple noise draws as well, so RAM's edge over them comes from its regression target: a closed-form adjoint-matching condition, whereas DiffusionNFT relies on a positive/negative-guidance surrogate and AWM on a flow-matching ELBO.
The AWM numbers we report are lower than in the original publication, which uses $72$ prompts per training step.
We use $48$ uniformly across all methods for a fair comparison.

\section{Related Work}
\label{sec:related_work}

\paragraph{Diffusion and flow-matching pretraining.}
Diffusion models~\citep{sohl2015deep,ho2020ddpm} and the score-based SDE formulation~\citep{song2020scorebased} established regression against analytically noised data as the standard pretraining recipe.
Stochastic interpolants and flow matching unify deterministic flows and stochastic diffusions under the same velocity regression template~\citep{albergo2023stochastic,lipman2022flowmatching}, and rectified flow~\citep{liu2022rectifiedflow,esser2024scaling} has become the default parametrization for large-scale image and video models.
RAM reuses this regression structure, now with a target built from the current model and a scalar reward.

\paragraph{Policy-gradient RL post-training.}
Trajectory-level policy gradients were first applied to diffusion RL post-training by~\citet{black2023training}, with KL-regularized variants in DPOK~\citep{fan2023dpok}.
Subsequent work scales this view to large text-to-image models and to flow matching, including Flow-GRPO~\citep{liu2025flowgrpo} and continuous-time score-as-control formulations~\citep{zhao2025scoreasaction}, with further refinements through dense credit assignment or exploration~\citep{deng2024prdp,chae2025diffexp,yang2024densereward,liu2024efficient}.
Preference-supervised variants replace scalar rewards with pairwise comparisons~\citep{wallace2024diffusiondpo,yang2024d3po,kim2024preferencealignmentflowmatching}.
All of these treat denoising as a generic Markov decision process and learn from stochastic rollouts, with noise schedules tuned empirically or hybridized with ODE solvers~\citep{deng2026densegrpo}. None exploit the analytic noising structure that makes pretraining cheap.
RAM stays in the black-box reward setting, but replaces the policy-gradient rollout with a pretraining-like regression against a closed-form target.

\paragraph{Stochastic optimal control and adjoint methods.}
KL-regularized reward maximization in continuous time is naturally an entropy-regularized stochastic optimal control problem~\citep{todorov2006linearly}.
ELEGANT~\citep{uehara2024fine} develops this perspective for diffusion post-training, with an emphasis on avoiding reward collapse. Adjoint Matching~\citep{domingo2024adjoint} introduces the memoryless noise schedule that turns the Bellman fixed point into an on-policy regression of the adjoint.
Adjoint Sampling~\citep{havens2025adjointsamplinghighlyscalable} extends this perspective beyond RL post-training to sampling from unnormalized densities. Like RAM, it avoids SDE rollouts during training by sampling an endpoint and noising it analytically.
These methods target the KL-regularized objective directly, but their estimators require reward gradients and an adjoint ODE integrated backward along a stored SDE trajectory.
RAM sits most directly in the Adjoint Matching line and sharpens its path-space picture: the optimum tilts only the clean-endpoint distribution while preserving the conditional noising law given that endpoint. This justifies endpoint sampling with analytic noising, and yields a closed-form Bayes bridge target in place of a pathwise adjoint sweep.
This is what makes RAM substantially more efficient than Adjoint Matching: training avoids both the SDE rollout and the backward adjoint sweep.
The closely related Tilt Matching~\citep{potaptchik2025tilt} arrives at the same equation for how the velocity field changes as the reward tilt increases, and is likewise free of reward gradients and trajectory backpropagation. It reaches the tilted target by annealing through intermediate models. Our adjoint-matching formulation instead targets the fixed point directly.

\paragraph{Pretraining-like regression objectives.}
A recent line aligns RL post-training with the supervised structure of pretraining.
DiffusionNFT~\citep{zheng2025diffusionnft} defines an improvement direction by contrasting positive and negative endpoint distributions, then distills the resulting guidance into a single model with tunable strength.
AWM~\citep{xue2025awm} starts from a GRPO objective over clean endpoints, replaces the intractable sequence likelihood by a score/flow-matching ELBO, and adds a velocity-space KL penalty.
\citet{choi2026rethinking} further argue empirically that the likelihood estimator matters more than the particular outer loss.
RAM instead starts from the KL-regularized control problem, giving the update a direct control interpretation rather than a guidance-distillation or likelihood-surrogate interpretation. The principled derivation also yields a simpler objective in practice.

\paragraph{Differentiable-reward methods.}
When rewards are differentiable, end-to-end backpropagation through the sampler provides a direct training signal.
ImageReward introduces a learned human-preference scorer and the ReFL algorithm~\citep{xu2023imagereward}. AlignProp~\citep{prabhudesai2023alignprop}, DRaFT~\citep{clark2023directly}, and DRTune~\citep{wu2024drtune} explore full- or partial-chain reward backpropagation and the resulting memory/stability trade-offs.
These methods rely on reward gradients backpropagated through the sampling chain or a truncated segment of it. RAM instead targets black-box rewards and keeps the KL-regularized control derivation.

\paragraph{Bridge estimators for controlled processes.}
Recent work estimates bridge scores under general controlled processes using Malliavin calculus or control self-consistency~\citep{pidstrigach2025conditioning,howard2025controlconsistencylossesdiffusion}, building on classical Malliavin representations of scores and sensitivities~\citep{lehec2013representation,baudoin2002conditioned}.
The path-cost-corrected estimators we introduce in~\Cref{sec:path_cost} connect to this literature, while the main RAM objective specializes to a cheaper Bayes bridge approximation tailored to large-scale RL post-training.

\paragraph{Flow-map reward alignment.}
Recently, reward alignment has been applied to flow maps~\citep{boffi2024flowmap}, which enable few-step generation. Meta Flow Maps~\citep{potaptchik2026meta} and Diamond Maps~\citep{holderrieth2026diamond} align such models to a reward, but require reward gradients. We believe that extending RAM to flow maps, while keeping its gradient-free regression structure, is a promising direction for future work.

\section*{Acknowledgments}

We acknowledge support from the Alexander von Humboldt Foundation. We thank the Leibniz Supercomputing Centre (LRZ) for providing computational resources, and Lennart Redl for feedback on the writing.

\bibliographystyle{plainnat}
\bibliography{references}
\clearpage

\appendix
\section{Optimal-Control Foundations}
\label{app:proofs}

\subsection{Optimal control as path-space KL regularization}
\label{app:ocp_as_kl}

Let $\mathbb P$ denote the path law of the reference backward process~\eqref{eq:backward_process} and $\mathbb P^u$ the path law of the controlled process~\eqref{eq:controlled_process}.
We assume Novikov's condition, a moment bound on~$u$ ensuring that the change of measure below is well defined, and $\E_{\mathbb P}[\exp(r(X_0))] < \infty$.
Girsanov's theorem relates the path laws of two diffusions that share the same noise coefficient but differ in drift; applied here it gives
\begin{equation}
  \KL(\mathbb P^u  \|  \mathbb P)
  =
  \frac{1}{2} \E_{\mathbb P^u}\!\left[
    \int_0^1 \|u_t(X_t^u)\|^2\,\mathd t
  \right].
  \label{eq:path_kl_equals_energy}
\end{equation}
The stochastic optimal control problem~\eqref{eq:SOC} is therefore equivalent to the Gibbs variational problem
\begin{equation}
  \sup_{\mathbb Q \ll \mathbb P}
  \bigl\{
    \E_{\mathbb Q}[r(X_0)]
    -
    \KL(\mathbb Q  \|  \mathbb P)
  \bigr\},
  \label{eq:gibbs_variational_problem}
\end{equation}
where $\mathbb Q \ll \mathbb P$ denotes absolute continuity, so every trajectory possible under~$\mathbb Q$ must also have positive probability under~$\mathbb P$.
The closed-form solution is the tilted path law with density ratio
\begin{equation}
  \frac{\mathd \mathbb Q^\ast}{\mathd \mathbb P}
  =
  \frac{\exp(r(X_0))}{
    \E_{\mathbb P}\!\left[\exp(r(X_0))\right]}.
  \label{eq:optimal_path_measure}
\end{equation}

\subsection{Path-law factorization}
\label{app:path_law_factorization}

\begin{lemma}[Path-law factorization]
  \label{lem:path_factorization}
  Let $\Pi(\cdot \mid x_0)$ denote the conditional law of the reference backward trajectory $(X_t)_{t \in [0,1]}$ given $X_0 = x_0$.
  Then
  \begin{equation}
    \mathbb P
    =
    \int \Pi(\cdot \mid x_0) p_0^{\mathrm{ref}}(x_0)\,\mathd x_0,
    \label{eq:path_factorization}
  \end{equation}
  and $\Pi(\cdot \mid x_0)$ does not depend on the data distribution.
\end{lemma}

\begin{proof}
Let $\Pi^{\mathrm{fwd}}(\cdot \mid x_0)$ be the conditional law of the forward process~\eqref{eq:forward_process} given $X_0 = x_0$.
Because the coefficients of~\eqref{eq:forward_process} are deterministic, $\Pi^{\mathrm{fwd}}(\cdot \mid x_0)$ does not depend on the data distribution, and the forward path law factors as $\int \Pi^{\mathrm{fwd}}(\cdot \mid x_0) p_0^{\mathrm{ref}}(x_0)\,\mathd x_0$.
The backward process~\eqref{eq:backward_process} is the exact time reversal of the forward process~\citep{anderson1982reverse}, so reversing trajectories yields a kernel $\Pi(\cdot \mid x_0)$, obtained from $\Pi^{\mathrm{fwd}}$ by time reversal, for which~\eqref{eq:path_factorization} holds.
Time reversal preserves the data-independence.
\end{proof}

This is the path-space factorization underlying what \citet{domingo2024adjoint} call the memoryless property: the conditional trajectory law given~$X_0$ carries no information about which data distribution was used.
At $t{=}1$, the noising kernel~\eqref{eq:noising_kernel} reduces to $\mathcal N(0,\,I)$ regardless of $X_0$, so the source marginal is also data-independent.

\subsection{Optimal controlled process}
\label{app:optimal_controlled_process}

We prove \Cref{thm:optimal_controlled_process}.

\begin{proof}
We prove the two directions of~\eqref{eq:optimal_process} separately.

\paragraph{Optimality $\Rightarrow$ form.}
Let $\mathbb Q^\ast$ denote the optimal path law.
By \eqref{eq:optimal_path_measure}, $\mathbb Q^\ast$ tilts $\mathbb P$ by $\exp(r(X_0))/Z$ with $Z = \E_{\mathbb P}[\exp(r(X_0))]$.
Substituting the factorization of \Cref{lem:path_factorization},
\begin{equation*}
  \mathbb Q^\ast
  =
  \int \Pi(\cdot \mid x_0)
  \frac{\exp(r(x_0))}{Z} p_0^{\mathrm{ref}}(x_0)\,\mathd x_0
  =
  \int \Pi(\cdot \mid x_0) p_0^\ast(x_0)\,\mathd x_0,
\end{equation*}
with $p_0^\ast(x) \propto p_0^{\mathrm{ref}}(x)\exp(r(x))$.
Since the tilt depends only on~$x_0$, it changes only the endpoint marginal to~$p_0^\ast$ while leaving the conditional path law $\Pi(\cdot \mid x_0)$ unchanged, which is~\eqref{eq:optimal_process}.
Equivalently, the optimal process is the time reversal of the forward process started from~$p_0^\ast$, so applying the noising kernel gives the optimal time-$t$ marginal
\begin{equation*}
  p_t^\ast
  =
  \Law\!\bigl((1-t)\,X_0 + t \epsilon\bigr),
  \qquad
  X_0 \sim p_0^\ast,
  \quad
  \epsilon \sim \mathcal N(0,\,I).
\end{equation*}

\paragraph{Form $\Rightarrow$ optimality.}
Conversely, suppose~$X^u$ is a controlled process satisfying~\eqref{eq:optimal_process}.
Integrating the matching conditional law against~$p_0^\ast$ and applying \Cref{lem:path_factorization},
\begin{equation*}
  \mathbb P^u
  =
  \int \Pi(\cdot \mid x_0) p_0^\ast(x_0)\,\mathd x_0.
\end{equation*}
The reference factorization $\mathbb P = \int \Pi(\cdot \mid x_0) p_0^{\mathrm{ref}}(x_0)\,\mathd x_0$ shares the same conditional kernel, so the Radon--Nikodym derivative depends only on the endpoint,
\begin{equation*}
  \frac{\mathd \mathbb P^u}{\mathd \mathbb P}(\omega)
  =
  \frac{p_0^\ast(X_0(\omega))}
       {p_0^{\mathrm{ref}}(X_0(\omega))}
  =
  \frac{\exp(r(X_0(\omega)))}{Z},
\end{equation*}
which is exactly the Gibbs-optimal density ratio~\eqref{eq:optimal_path_measure}, so $\mathbb P^u = \mathbb Q^\ast$.
The Gibbs problem~\eqref{eq:gibbs_variational_problem} has a strictly concave objective and hence a unique maximizer, so~$u$ solves~\eqref{eq:SOC}.
\end{proof}

\subsection{Optimal velocity and score identities}
\label{app:proof_optimal_score_corollary}

We verify the velocity and score identities stated after \Cref{thm:optimal_controlled_process}, and derive the control--velocity relation~\eqref{eq:control_to_velocity}.

\begin{proof}
By \Cref{thm:optimal_controlled_process}, the optimal time-$t$ marginal equals that of the forward process started from $p_0^\ast$, so
\begin{equation*}
  v_t^\ast(x)
  =
  \E\!\left[
    \epsilon - X_0
    \;\middle|\;
    X_t = x
  \right],
  \qquad
  X_0 \sim p_0^\ast,
\end{equation*}
which has the same functional form as~\eqref{eq:velocity_field} with $p_0^\ast$ in place of $p_0$.
The velocity--score relation~\eqref{eq:score_from_velocity} carries over with $p_t^\ast$ in place of~$p_t$ by the same derivation.

For the control--velocity relation, compare the drift decompositions of the reference and optimal processes,
\begin{align*}
  b_t^{\mathrm{ref}}(x)
  &=
  v_t^{\mathrm{ref}}(x)
  -
  \frac{\sigma_t^{2}}{2} \nabla_x\!\log p_t^{\mathrm{ref}}(x),
  \\
  b_t^{\mathrm{ref}}(x) + \sigma_t u_t^\ast(x)
  &=
  v_t^\ast(x)
  -
  \frac{\sigma_t^{2}}{2} \nabla_x\!\log p_t^\ast(x).
\end{align*}
Subtracting gives
\begin{equation*}
  \sigma_t u_t^\ast
  =
  \bigl(v_t^\ast - v_t^{\mathrm{ref}}\bigr)
  -
  \frac{\sigma_t^{2}}{2}
  \bigl(
    \nabla_x\!\log p_t^\ast
    -
    \nabla_x\!\log p_t^{\mathrm{ref}}
  \bigr).
\end{equation*}
Substituting $\nabla_x\!\log p_t = 2(\kappa_t x - v_t)/\sigma_t^{2}$ for both marginals, the $\kappa_t x$ terms cancel and the score difference is $-2(v_t^\ast - v_t^{\mathrm{ref}})/\sigma_t^{2}$, so $\sigma_t u_t^\ast = 2(v_t^\ast - v_t^{\mathrm{ref}})$.
\end{proof}

\subsection{Verification and self-consistency}
\label{app:verification_fixed_points}
\label{proof:consistency}

We prove that the self-consistency condition $u_t = -\sigma_t A_t^u$~\eqref{eq:value_function_matching} is both necessary and sufficient for optimality.
One can thereby certify a candidate control by checking a single PDE rather than comparing against all alternatives.

The result is standard HJB / adjoint-matching theory. We include the proof to keep the paper self-contained.

\begin{theorem}[Verification]
  \label{thm:verification}
  Assume the Hamilton--Jacobi--Bellman (HJB) equation associated with~\eqref{eq:SOC} admits a unique $C^{1,2}$ solution (once continuously differentiable in~$t$, twice in~$x$) and that the optimal control is a Markov feedback, depending on the current state~$x$ only, not on the trajectory history.
  Let~$u$ be an admissible control with finite expected cost whose value function~$V_t^u$ is $C^{1,2}$ in~$(t,x)$.
  Then
  \begin{equation}
    u \text{ is optimal}
    \qquad\Longleftrightarrow\qquad
    u_t(x) = -\sigma_t \nabla_x V_t^u(x)
    \quad \text{for all }(t,x),
    \label{eq:consistency_condition_app}
  \end{equation}
  and in that case $X_0^u \sim p_0^\ast$.
\end{theorem}

\begin{proof}
The proof rests on two PDEs the value function can satisfy: a \emph{linear} one that holds for any control, and a \emph{nonlinear} one (the HJB equation) that singles out the optimum.
The self-consistency condition is exactly what turns the first into the second.

\paragraph{Value-function PDE.}
For any admissible control~$u$, the value function~$V_t^u$ satisfies the linear PDE
\begin{equation}
  0
  =
  -\partial_t V_t^u(x)
  +
  \mathcal B_t^u V_t^u(x)
  -
  \frac{1}{2}\|u_t(x)\|^2,
  \qquad
  V_0^u(x) = r(x),
  \label{eq:linear_pde_app}
\end{equation}
where $\mathcal B_t^u$ is the backward transition generator of the controlled diffusion~\eqref{eq:controlled_process}, encoding how the reverse-time drift and noise act on a test function~$f$,
\begin{equation*}
  \mathcal B_t^u f(x)
  =
  -\bigl(b_t^{\mathrm{ref}}(x) + \sigma_t u_t(x)\bigr)^\top
  \nabla_x f(x)
  +
  \frac{\sigma_t^{2}}{2} \Delta_x f(x).
\end{equation*}
The first term captures transport by the drift, the second captures diffusion.
Equation~\eqref{eq:linear_pde_app} is linear in~$V_t^u$ because~$u$ is fixed; it describes how the expected future value evolves along the controlled process.

\paragraph{Self-consistency $\Rightarrow$ optimality.}
Suppose $u_t = -\sigma_t \nabla_x V_t^u$.
Substituting this into the generator, the controlled-drift term expands as $\sigma_t^{2} \|\nabla_x V_t^u\|^2$, while the running cost is $\tfrac{1}{2}\|u_t\|^2 = \tfrac{\sigma_t^{2}}{2}\|\nabla_x V_t^u\|^2$.
Substituting both into~\eqref{eq:linear_pde_app} cancels half of the quadratic terms and leaves
\begin{equation}
  0
  =
  -\partial_t V_t^u
  -
  (b_t^{\mathrm{ref}})^\top \nabla_x V_t^u
  +
  \frac{\sigma_t^{2}}{2} \Delta_x V_t^u
  +
  \frac{\sigma_t^{2}}{2} \|\nabla_x V_t^u\|^2,
  \label{eq:hjb_app}
\end{equation}
which is the HJB equation for~\eqref{eq:SOC}.
Since the HJB equation has a unique solution, $V^u$ must equal the optimal value function, so~$u$ is optimal.

\paragraph{Optimality $\Rightarrow$ self-consistency.}
Let $u^\ast$ be an optimal Markov control with value function $V = V^{u^\ast}$.
The HJB equation can be written as
\begin{equation*}
  0
  =
  -\partial_t V
  -
  (b_t^{\mathrm{ref}})^\top \nabla_x V
  +
  \frac{\sigma_t^{2}}{2} \Delta_x V
  +
  \sup_{a \in \R^d}
  \bigl\{
    -\sigma_t a^\top \nabla_x V
    -
    \tfrac{1}{2}\|a\|^2
  \bigr\},
  \qquad
  V_0 = r.
\end{equation*}
The supremum is a concave quadratic with unique maximizer $a^\ast = -\sigma_t \nabla_x V$.
Since~$u^\ast$ attains the supremum, $u_t^\ast(x) = -\sigma_t \nabla_x V_t^{u^\ast}(x)$.
The endpoint law follows from \Cref{thm:optimal_controlled_process}.
\end{proof}

\section{Derivations for RAM}
\label{app:ram_derivations}

This appendix collects the derivations behind \Cref{sec:ram}.

\subsection{Bayes bridge score derivation}
\label{app:bayes_bridge_score}

We derive~\eqref{eq:general_bayes_bridge_score} in the general form covering any $0 \le s < t \le 1$.
The forward SDE~\eqref{eq:forward_process} is linear, so its transition kernel from~$s$ to~$t$ is Gaussian:
\begin{equation}
  p_{t|s}(x_t \mid x_s)
  =
  \mathcal N\!\bigl(a_{t,s} x_s,\,\beta_{t,s}^{2} I\bigr),
  \qquad
  a_{t,s} = \frac{1-t}{1-s},
  \qquad
  \beta_{t,s}^{2} = t^{2} - \frac{(1-t)^{2} s^{2}}{(1-s)^{2}}.
  \label{eq:forward_kernel_app}
\end{equation}
Bayes' rule gives
\begin{equation*}
  \nabla_{x_t}\!\log p_{s|t}(x_s \mid x_t)
  =
  \nabla_{x_t}\!\log p_{t|s}(x_t \mid x_s)
  -
  \nabla_{x_t}\!\log p_t(x_t).
\end{equation*}
The Gaussian term is $\nabla_{x_t}\!\log p_{t|s}(x_t \mid x_s) = -(x_t - a_{t,s} x_s)/\beta_{t,s}^{2}$, and the marginal score follows from~\eqref{eq:score_from_velocity} as $\nabla_x\!\log p_t(x) = -(x + (1-t) v_t(x))/t$.
Combining the two,
\begin{equation}
  \nabla_{x_t}\!\log p_{s|t}(x_s \mid x_t)
  =
  -\frac{x_t - a_{t,s} x_s}{\beta_{t,s}^{2}}
  +
  \frac{x_t + (1-t) v_t(x_t)}{t}.
  \label{eq:generalized_bayes_bridge_score}
\end{equation}
This is the Bayes bridge score used by both the jump estimator~\eqref{eq:jump_estimator} and the pathwise family of~\Cref{sec:path_cost}.
Setting $s{=}0$ gives $a_{t,0} = 1-t$ and $\beta_{t,0} = t$. With $\epsilon \coloneqq (x_t - (1-t) x_0)/t$, equation~\eqref{eq:generalized_bayes_bridge_score} simplifies to the prefactor form $\frac{1-t}{t}\bigl(v_t(x_t) - (\epsilon - x_0)\bigr)$ used in~\eqref{eq:general_bayes_bridge_score}.

\subsection{Control-space form of RAM and time-dependent weighting}
\label{app:ram_control_space}

The Bayes bridge score~\eqref{eq:general_bayes_bridge_score} has prefactor $(1-t)/t$, which combines with $\sigma_t^{2} = 2t/(1-t)$ via the identity $\sigma_t (1-t)/t = 2/\sigma_t$.

Substituting the reward proxy~\eqref{eq:ram_reward_proxy_identity} and the bridge score~\eqref{eq:general_bayes_bridge_score} into the self-consistency condition $u_t = -\sigma_t A_t^u$ gives the control-space target
\begin{equation*}
  u_t
  =
  \mathrm{sg}\!\left(
    \frac{2}{\sigma_t}\, r(X_0) \bigl(
      (\epsilon - X_0)
      -
      v_t^\theta(X_t)
    \bigr)
  \right).
\end{equation*}
Regressing $u_t^\theta = (2/\sigma_t)(v_t^\theta - v_t^{\mathrm{ref}})$ against this target, the common $2/\sigma_t$ factor pulls out as $4/\sigma_t^{2}$ on the velocity-space squared residual:
\begin{equation}
  \mathcal L_u(\theta)
  =
  \E\!\left[
    \frac{4}{\sigma_t^{2}}
    \left\|
      v_t^\theta(X_t)
      -
      \mathrm{sg}\!\left(
        v_t^{\mathrm{ref}}(X_t)
        +
        r(X_0)
        \bigl(
          (\epsilon - X_0)
          -
          v_t^\theta(X_t)
        \bigr)
      \right)
    \right\|^2
  \right].
  \label{eq:ram_weighted_velocity_loss}
\end{equation}
The prefactor $4/\sigma_t^{2} = 2(1-t)/t$ downweights timesteps near the source.
Dropping it yields the uniformly weighted velocity loss used in the main text.

\subsection{Path-integral characterization of the optimum}
\label{app:proof_tilt_path}

We prove \Cref{lem:tilt_path}. Throughout, $\epsilon \sim \mathcal{N}(0,\,I)$ is independent of $X_0$ and $X_t = (1-t) X_0 + t \epsilon$. The distribution of $X_0$ is stated alongside each expectation.

\paragraph{The ODE.}
We show that
\begin{equation*}
  \partial_\lambda v_t^\lambda(x)
  =
  \E\!\left[ r(X_0)\bigl((\epsilon - X_0) - v_t^\lambda(x)\bigr) \;\middle|\; X_t = x \right],
  \qquad X_0 \sim p_0^\lambda,
\end{equation*}
which is~\eqref{eq:tilt_velocity_ode}. Starting from $v_t^\lambda(x) = \E[\epsilon - X_0 \mid X_t = x]$ with $X_0 \sim p_0^\lambda$, we differentiate in $\lambda$. The score in $\lambda$ of the conditional density is
\begin{equation*}
  \partial_\lambda \log p_0^\lambda(x_0 \mid X_t = x)
  =
  r(x_0) - \E\!\left[r(X_0) \;\middle|\; X_t = x\right]_{X_0 \sim p_0^\lambda},
\end{equation*}
so for $\lambda$-independent $f$,
\begin{equation*}
  \partial_\lambda
  \E\!\left[ f \;\middle|\; X_t = x \right]_{X_0 \sim p_0^\lambda}
  =
  \E\!\left[ r(X_0) f \;\middle|\; X_t = x \right]_{X_0 \sim p_0^\lambda}
  -
  \E\!\left[ f \;\middle|\; X_t = x \right]_{X_0 \sim p_0^\lambda}
  \E\!\left[ r(X_0) \;\middle|\; X_t = x \right]_{X_0 \sim p_0^\lambda}.
\end{equation*}
Applying with $f = \epsilon - X_0$ and using $\E[\epsilon - X_0 \mid X_t = x] = v_t^\lambda(x)$ under $p_0^\lambda$ yields the claim.

\paragraph{The integral form.}
Integrating the ODE from $\lambda = 0$ to $\lambda = 1$ and using $v_t^0 = v_t^{\mathrm{ref}}$, $v_t^1 = v_t^\ast$ gives
\begin{equation*}
  v_t^\ast(x) - v_t^{\mathrm{ref}}(x)
  =
  \int_0^1
  \E\!\left[ r(X_0)\bigl((\epsilon - X_0) - v_t^\lambda(x)\bigr) \;\middle|\; X_t = x \right]
  \mathd\lambda,
  \qquad X_0 \sim p_0^\lambda,
\end{equation*}
which is~\eqref{eq:exact_tilt_integral}.

\subsection{Exact case: Gaussian reference under a linear reward}
\label{app:gaussian_case}

We show that for $p_0^{\mathrm{ref}}$ Gaussian and $r$ linear, the velocity-field family $\{v_t^\lambda\}_{\lambda \in [0,1]}$ from~\Cref{lem:tilt_path} is affine in~$\lambda$. The integrand of~\eqref{eq:exact_tilt_integral} is then constant in~$\lambda$, the right-endpoint rule is exact, and the RAM target~\eqref{eq:ram_reward_proxy_identity} matches the optimal adjoint at every~$x$.

Let $p_0^{\mathrm{ref}} = \mathcal N(\mu, \Sigma)$ with positive-definite~$\Sigma$ and $r(x) = b^\top x + c$. Completing the square in the log-density of $p_0^\lambda \propto p_0^{\mathrm{ref}}\, e^{\lambda r}$ gives
\begin{equation}
  p_0^\lambda
  =
  \mathcal N(\mu_\lambda,\, \Sigma),
  \qquad
  \mu_\lambda \coloneqq \mu + \lambda \Sigma b,
  \label{eq:gaussian_tilted_endpoint}
\end{equation}
a Gaussian whose mean shifts linearly in~$\lambda$ and whose covariance does not depend on~$\lambda$. Under the analytic noising kernel~\eqref{eq:noising_kernel}, the joint $(X_0, X_t)$ is Gaussian, and the conditional law of $X_0$ given $X_t = x$ is
\begin{equation}
  X_0 \mid X_t = x
  \;\sim\;
  \mathcal N\!\bigl(m_t^\lambda(x),\, S_t\bigr),
  \label{eq:gaussian_bridge_moments}
\end{equation}
with mean
\begin{equation*}
  m_t^\lambda(x)
  =
  \mu_\lambda + J_t (x - (1-t) \mu_\lambda),
  \qquad
  J_t \coloneqq (1-t)\,\Sigma\bigl((1-t)^2 \Sigma + t^2 I\bigr)^{-1},
\end{equation*}
and covariance $S_t = t^2 \bigl((1-t)^2 \Sigma + t^2 I\bigr)^{-1} \Sigma$ that does not depend on~$\lambda$. The conditional mean $m_t^\lambda(x)$ is affine in~$\lambda$.

Using $\epsilon = (X_t - (1-t) X_0)/t$, the velocity field~\eqref{eq:velocity_field} for endpoint $p_0^\lambda$ evaluates as
\begin{equation*}
  v_t^\lambda(x)
  =
  -\E[X_0 \mid X_t = x]
  +
  \frac{1}{t}\bigl(x - (1-t)\, \E[X_0 \mid X_t = x]\bigr)
  =
  \frac{x - m_t^\lambda(x)}{t},
\end{equation*}
which is affine in $m_t^\lambda(x)$ and therefore affine in~$\lambda$. Substituting into~\eqref{eq:tilt_velocity_ode} and applying~\Cref{lem:tilt_path} gives the claim.

A concrete consequence is that $v_t^\lambda(x) - v_t^{\mathrm{ref}}(x) = -\tfrac{1}{t}\bigl(I - (1-t) J_t\bigr)\,\lambda \Sigma b$ is independent of~$x$. Through the control--velocity relation~\eqref{eq:control_to_velocity}, the optimal control $u_t^\ast$ is therefore independent of position, recovering the path-cost-vanishing instance referenced at~\eqref{eq:ram_reward_proxy_identity}.

\section{Derivations for Retaining the Path-Cost Correction}
\label{app:path_cost_derivations}

This appendix collects the derivations behind \Cref{sec:path_cost}.
We begin with the endpoint-sampling jump estimator, then turn to the pathwise family.

\subsection{Random-jump identity and the jump estimator}
\label{app:endpoint_jump_derivation}

The jump estimator~\eqref{eq:jump_estimator} combines a single-point control-cost evaluation with a Bayes bridge score.
Both pieces are justified by the following lemma.

\begin{lemma}[Unbiased single-point cost estimate]
  \label{lem:random_jump}
  Let $s \sim \mathcal U[0,t)$ be independent of the trajectory.
  Then
  \begin{equation}
    \E\!\left[
      \frac{t}{2} \|u_s(X_s^u)\|^2
      \;\middle|\;
      X_t^u = x
    \right]
    =
    \frac{1}{2} \E\!\left[
      \int_0^t \|u_\tau(X_\tau^u)\|^2\,\mathd\tau
      \;\middle|\;
      X_t^u = x
    \right],
    \label{eq:random_jump_identity}
  \end{equation}
  and consequently
  \begin{equation}
    \E\!\left[
      r(X_0^u) - \frac{t}{2} \|u_s(X_s^u)\|^2
      \;\middle|\;
      X_t^u = x
    \right]
    =
    V_t^u(x).
    \label{eq:single_point_value}
  \end{equation}
\end{lemma}

\begin{proof}
Conditioning on $X_t^u = x$ and applying Fubini's theorem,
\begin{equation*}
  \E\!\left[
    \frac{t}{2} \|u_s(X_s^u)\|^2
    \;\middle|\;
    X_t^u = x
  \right]
  =
  \frac{1}{2}
  \int_0^t
  \E\!\left[
    \|u_\tau(X_\tau^u)\|^2
    \;\middle|\;
    X_t^u = x
  \right]\mathd\tau,
\end{equation*}
which is~\eqref{eq:random_jump_identity}.
Subtracting from the reward term in the value function~\eqref{eq:value_function} gives~\eqref{eq:single_point_value}.
\end{proof}

We now derive the jump estimator.
Under \Cref{thm:optimal_controlled_process}, at initialization ($v^\theta = v^{\mathrm{ref}}$) and at the optimum, the marginals of the controlled process coincide with those of the forward process from the appropriate endpoint distribution.
The Markov property of the noising kernel then guarantees that the triple $(X_0, X_s, X_t)$ obtained by sampling $X_s \sim p_{s|0}(\cdot \mid X_0)$ and then $X_t \sim p_{t|s}(\cdot \mid X_s)$ has the same joint law as $(X_0^u, X_s^u, X_t^u)$ along the controlled process at those two points.

Conditional on $(X_0, X_s, X_t)$, the state $X_t$ enters the joint density only through the bridge $p_{s|t}^u$.
Combining this with~\eqref{eq:single_point_value} and the reasoning behind \eqref{eq:exact_adjoint_decomposition}, differentiating the value function in~$x$ produces
\begin{equation*}
  A_t^u(x)
  =
  \E\!\left[
    \Bigl(
      r(X_0^u) - \frac{t}{2} \|u_s(X_s^u)\|^2
    \Bigr)
    \nabla_{x_t}\!\log p_{s|t}^u(X_s^u \mid x_t)\big|_{x_t = x}
    \;\middle|\;
    X_t^u = x
  \right],
\end{equation*}
with the expectation over $s \sim \mathcal U[0,t)$ and the trajectory.
Replacing the controlled-process sample by the two-step construction above and substituting the Bayes bridge score~\eqref{eq:generalized_bayes_bridge_score} yields~\eqref{eq:jump_estimator}.

\subsection{Generalized adjoint identity}
\label{app:proof_generalized_ram}

We prove \Cref{thm:generalized_ram}.

\begin{proof}
Fix an admissible control~$u$ and $0 \le s < t \le 1$; assume the regularity required to interchange differentiation and conditional expectation.
The value function satisfies the recursion
\begin{equation}
  V_t^u(x)
  =
  \E\!\left[
    V_s^u(X_s^u)
    -
    \frac{1}{2}\int_s^t \|u_\tau(X_\tau^u)\|^2\,\mathd\tau
    \;\middle|\;
    X_t^u = x
  \right].
  \label{eq:value_recursion_app}
\end{equation}
We differentiate both sides in~$x$.
For the prefix, the conditional law of $X_s^u$ given $X_t^u = x$ has density $p_{s|t}^u(\cdot \mid x)$, so the score-function identity gives
\begin{equation*}
  \nabla_x \E\!\left[
    V_s^u(X_s^u)
    \;\middle|\;
    X_t^u = x
  \right]
  =
  \E\!\left[
    V_s^u(X_s^u) \nabla_x\!\log p_{s|t}^u(X_s^u \mid x)
    \;\middle|\;
    X_t^u = x
  \right].
\end{equation*}
For the suffix, we differentiate pathwise through the rollout.
Subtracting gives the right-hand side of~\eqref{eq:generalized_adjoint}, equal to $\nabla_x V_t^u(x)$ by the left-hand side.
\end{proof}

From a single SDE rollout we estimate the right-hand side as
\begin{equation*}
\widehat A_{s,t}^u = \widehat V_s\, \widehat S_{s,t} - \tfrac{1}{2}\, \widehat G_s,
\end{equation*}
where $\widehat V_s = r(X_0^u) - \tfrac{1}{2}\int_0^s \|u_\tau(X_\tau^u)\|^2\,\mathd\tau$ is the prefix value accumulated along the rollout, $\widehat S_{s,t}$ estimates the bridge score $\nabla_{x_t}\!\log p_{s|t}^u(X_s^u \mid x_t)$, and $\widehat G_s$ is the suffix path-cost gradient. The next subsections develop estimators for $\widehat S_{s,t}$ and the discrete-time machinery that produces $\widehat G_s$.

\subsection{Malliavin bridge-score estimator}
\label{app:malliavin_score_estimator}

We derive the Malliavin estimator $\widehat S_{s,t}$.

\begin{proposition}[Malliavin score estimator]
  \label{prop:malliavin_score}
  Let $(X_\tau^u)_{\tau \in [s,t]}$ be a segment of the controlled process~\eqref{eq:controlled_process}, and let $J_{\tau|s}^u$ denote the Jacobian flow
  \begin{equation}
    \mathd J_{\tau|s}^u
    =
    \nabla_x\!\bigl(
      b_\tau^{\mathrm{ref}} + \sigma_\tau u_\tau
    \bigr)(X_\tau^u) J_{\tau|s}^u\,\mathd\tau,
    \qquad
    J_{s|s}^u = I.
    \label{eq:jacobian_flow_app}
  \end{equation}
  Define
  \begin{equation}
    \widehat S_{s,t}
    \coloneqq
    \left(
      \int_s^t \sigma_\tau^{2}\,\mathd\tau
    \right)^{-1}
    \int_s^t (J_{\tau|s}^u)^\top \sigma_\tau\,\mathd B_\tau.
    \label{eq:malliavin_estimator_app}
  \end{equation}
  Then $\E\!\bigl[\widehat S_{s,t} \;\big|\; X_s^u, X_t^u\bigr] = \nabla_{x_t}\!\log p_{s|t}^u(X_s^u \mid x_t)|_{x_t = X_t^u}$.
\end{proposition}

\begin{proof}
We apply \citet[Proposition~2.4]{pidstrigach2025conditioning}: for a diffusion $\mathd X_\tau = f_\tau(X_\tau)\,\mathd\tau + \sigma_\tau\,\mathd B_\tau$ and any non-decreasing absolutely continuous $\alpha$ on $[s,t]$, the bridge score satisfies
\begin{equation}
  \nabla_{x_t}\!\log p_{s|t}(x_s \mid x_t)
  =
  \frac{1}{\alpha_t - \alpha_s} \E\!\left[
    \int_s^t \alpha'_\tau \sigma_\tau^{-1} J_{\tau|s}^\top\,\mathd B_\tau
    \;\middle|\;
    X_s = x_s,  X_t = x_t
  \right].
  \label{eq:pidstrigach_general}
\end{equation}
Applied to the controlled process with $\sigma_\tau$ as in \eqref{eq:controlled_process}, the choice $\alpha'_\tau = \sigma_\tau^{2}$ makes $\alpha_t - \alpha_s = \int_s^t \sigma_\tau^{2}\,\mathd\tau$ and gives the estimator~\eqref{eq:malliavin_estimator_app}.
\end{proof}

The weighting $\alpha'_\tau = \sigma_\tau^{2}$ is convenient because the integrand becomes $(J_{\tau|s}^u)^\top \sigma_\tau\,\mathd B_\tau$; in a discretized rollout this is exactly the realized solver noise.

In practice, rather than accumulating \eqref{eq:malliavin_estimator_app} forward in~$\tau$, we use the equivalent reverse-mode adjoint SDE indexed by a running time~$\tau$ and a fixed endpoint~$t$,
\begin{equation}
  \mathd S^{(t)}_\tau
  =
  -
  \nabla_x\!\bigl(
    b_\tau^{\mathrm{ref}} + \sigma_\tau u_\tau
  \bigr)(X_\tau^u)^\top S^{(t)}_\tau\,\mathd\tau
  -
  \sigma_\tau\,\mathd B_\tau,
  \qquad
  S^{(t)}_t = 0,
  \label{eq:generalized_stochastic_recursion}
\end{equation}
with $\widehat S_{s,t} = (\int_s^t \sigma_\tau^{2}\,\mathd\tau)^{-1} S^{(t)}_s$.
Only vector--Jacobian products through the drift are required, never explicit Jacobians.
The same dual-indexed form applies to the suffix running-cost gradient,
\begin{equation}
  \mathd G^{(t)}_\tau
  =
  -
  \nabla_x\!\bigl(
    b_\tau^{\mathrm{ref}} + \sigma_\tau u_\tau
  \bigr)(X_\tau^u)^\top G^{(t)}_\tau\,\mathd\tau
  -
  \nabla_x\!\|u_\tau(X_\tau^u)\|^2\,\mathd\tau,
  \qquad
  G^{(t)}_t = 0,
  \label{eq:generalized_path_recursion}
\end{equation}
with $\widehat G_s \coloneqq G^{(t)}_s$.
The full-horizon specialization in~\Cref{app:adjoint_recursions} produces all endpoint-indexed accumulators $S^{(t_i)}_0, G^{(t_i)}_0$ in a single backward sweep.

\subsection{Local one-step Gaussian score}
\label{app:local_gaussian_score}

Let $s = t{-}\delta$ and write $b_t^u \coloneqq b_t^{\mathrm{ref}} + \sigma_t u_t$.
A single Euler--Maruyama step of the controlled SDE~\eqref{eq:controlled_process} from~$t$ to~$t{-}\delta$ gives
\begin{equation*}
  X_{t-\delta}^u
  =
  X_t^u - b_t^u(X_t^u) \delta + \Delta W_t,
  \qquad
  \Delta W_t \sim \mathcal N\!\bigl(0,\,\sigma_t^{2}\delta I\bigr).
\end{equation*}
Conditional on $X_t^u = x$, this defines the Gaussian one-step transition $p_{t-\delta|t}^{u,\delta}(\cdot \mid x) = \mathcal N(x - b_t^u(x) \delta,\,\sigma_t^{2}\delta I)$, with log-density $-\|X_{t-\delta}^u - \mu(x)\|^2/(2\sigma_t^{2}\delta)$ up to a constant, where $\mu(x) = x - b_t^u(x) \delta$.
Differentiating in~$x$,
\begin{equation}
  \nabla_x\!\log p_{t-\delta|t}^{u,\delta}(X_{t-\delta}^u \mid x)
  =
  \frac{
    \bigl(I - \delta \nabla_x b_t^u(x)\bigr)^\top \Delta W_t
  }{\sigma_t^{2} \delta},
  \label{eq:local_gaussian_score}
\end{equation}
which, evaluated at $x = X_t^u$, is the one-step bridge score for the local estimator.
Its difference from the exact continuous-time bridge score is the $O(\delta)$ error already incurred by Euler--Maruyama.

\subsection{Cancellation of $\nabla_x u_t$ terms in the local estimator}
\label{app:lean_cancellation}

We show that the $\nabla_x u_t$ contributions from the score and the path-cost gradient cancel to first order in~$\delta$, yielding the local target~\eqref{eq:lean_target}.
We work at the fixed point $u_t = -\sigma_t \nabla_x V_t^u$.

Let $s = t{-}\delta$.
From~\eqref{eq:local_gaussian_score} with $\Delta W_t = \sigma_t \Delta B_t$,
\begin{equation}
  \nabla_{x_t}\!\log p_{s|t}^{u,\delta}(X_s^u \mid x_t)
  \big|_{x_t = X_t^u}
  =
  \frac{
    \bigl(I - \delta \nabla_x b_t^u(X_t^u)\bigr)^\top \Delta W_t
  }{\sigma_t^{2} \delta}.
  \label{eq:lean_local_score}
\end{equation}
Writing $\nabla_x b_t^u = \nabla_x b_t^{\mathrm{ref}} + \sigma_t \nabla_x u_t$, the part of the REINFORCE term involving $\nabla_x u_t$ is
\begin{equation*}
  -V_s^u(X_s^u)
  (\nabla_x u_t(X_t^u))^\top \Delta B_t.
\end{equation*}
The path-cost term on a single step is
\begin{equation*}
  -\frac{1}{2} \nabla_x\!\int_s^t \|u_\tau(X_\tau^u)\|^2\,\mathd\tau
  =
  -\delta (\nabla_x u_t(X_t^u))^\top u_t(X_t^u)
  +
  O(\delta^2).
\end{equation*}

To compare the two, expand the value function backward from~$t$ to~$s$.
At the fixed point, It\^o's formula combined with $u_t = -\sigma_t \nabla_x V_t^u$ gives
\begin{equation*}
  V_s^u(X_s^u)
  =
  V_t^u(X_t^u)
  -
  u_t(X_t^u)^\top \Delta B_t
  +
  O(\delta).
\end{equation*}
Substituting this into the $\nabla_x u_t$ contribution of the score and taking conditional expectation given~$X_t^u$, the first expansion term vanishes by martingale property while the second yields $\delta (\nabla_x u_t)^\top u_t$, using $\E[\Delta B_t \Delta B_t^\top \mid X_t^u] = \delta I$.
This cancels the path-cost term to first order in~$\delta$.

What remains is the reference-drift piece,
\begin{equation*}
  \widehat A_t^{\mathrm{local}}
  \approx
  \E\!\left[
    \frac{V_{t-\delta}^u(X_{t-\delta}^u)}
         {\sigma_t^{2} \delta} \bigl(I - \delta \nabla_x b_t^{\mathrm{ref}}(X_t^u)\bigr)^\top
    \Delta W_t
    \;\middle|\;
    X_t^u
  \right].
\end{equation*}
Replacing the conditional value by the prefix estimate $\widehat V_{t-\delta}$ and simplifying yields the continuous-time local target
\begin{equation}
  \widehat A_t^{\mathrm{local}}
  \approx
  \frac{\widehat V_{t-\delta}}
       {\int_{t-\delta}^t \sigma_\tau^{2}\,\mathd\tau} \bigl(I - \delta \nabla_x b_t^{\mathrm{ref}}(X_t^u)\bigr)^\top
  \Delta W_t,
  \qquad
  \Delta W_t = \sigma_t\sqrt{\delta}\,\xi_t,\quad
  \xi_t \;\sim\; \mathcal N\!\bigl(0,\,I\bigr).
  \label{eq:lean_target}
\end{equation}
Since $\nabla_x b_t^{\mathrm{ref}} = 2\nabla_x v_t^{\mathrm{ref}} - \kappa_t I$, the local target requires only a VJP through the frozen reference model.

\subsection{Discrete-time full-horizon training procedure}
\label{app:adjoint_recursions}

We turn the full-horizon ($s{=}0$) instance of the estimator $\widehat A_{s,t}^u$ into a concrete training procedure.
Each iteration rolls out the controlled SDE once and then computes targets for all interior grid points in a single discrete reverse-mode sweep, the discrete analogue of solving the continuous adjoint equations backward from every endpoint~$t_i$ to~$0$.
Throughout this subsection, we write $X_i \coloneqq X_{t_i}$ and $b_i^{u^\theta} \coloneqq b_{t_i}^{u^\theta}$ for pointwise values at grid points.
Expressions of the form $\nabla_x(f(X_i)^\top v)$ are read with the gradient taken in~$x$ and evaluated at~$x{=}X_i$; this scalar-backprop form is what the implementation actually uses and never materializes the $d \times d$ Jacobian~$\nabla_x f$.

\paragraph{SDE rollout.}
Fix a grid $0 = t_0 < t_1 < \cdots < t_K < 1$ with $\Delta t_i \coloneqq t_i - t_{i-1}$, and draw independent innovations $\varepsilon_i \sim \mathcal N(0,\,I)$ for $i = 1,\ldots,K$.
Each innovation is scaled by the exact step standard deviation
\begin{equation}
  \sigma_i^{2}
  \coloneqq
  \int_{t_{i-1}}^{t_i} \sigma_\tau^{2}\,\mathd\tau,
  \label{eq:exact_step_variance}
\end{equation}
which evaluates to $\sigma_i^{2} = 2\log\!\frac{1-t_{i-1}}{1-t_i} - 2\Delta t_i$.
The integrated variance~\eqref{eq:exact_step_variance} is finite for $t_i < 1$ but diverges at the noise endpoint, since $\sigma_t^2 \sim 2/(1-t)$ as $t \uparrow 1$.
We therefore take $t_K < 1$ (in practice $t_K = 1 - \delta$ for small $\delta$).

Let $m_i^\theta$ denote the deterministic part of the reverse step $t_i \to t_{i-1}$ under the current model, and $m_i^{\mathrm{ref}}$ the corresponding reference step map.
Starting from $X_K \sim \mathcal N(0,\,I)$, we roll out
\begin{equation}
  X_{i-1}
  =
  m_i^\theta(X_i) + \sigma_i \varepsilon_i,
  \label{eq:discrete_step_map}
\end{equation}
and store $\{X_i\}_{i=0}^K$ together with $\{\varepsilon_i\}_{i=1}^K$.
Rather than regressing the continuous control~$u_i$ directly, the implementation regresses the exact normalized mean shift of this discrete step, $(m_i^\theta(x) - m_i^{\mathrm{ref}}(x))/\sigma_i$.
For a plain Euler step this reduces to $m_i^\theta(x) = x - \Delta t_i b_i^{u^\theta}(x)$. In our experiments we instead use an exact integrator that admits a closed-form $m_i^\theta$.

\paragraph{Adjoint sweep.}
Let $S_i$ and $G_i$ denote the discrete adjoint accumulators associated with the interval~$[0,t_i]$.
Although the loop runs in increasing~$i$, it still represents backward propagation through the reverse solver: extending the horizon from~$t_{i-1}$ to~$t_i$ transports the previously accumulated adjoint through the transpose Jacobian of the step map and adds the new contribution from step~$i$.
The recursion is
\begin{align}
  S_i
  &=
  \nabla_x\!\bigl(m_i^\theta(X_i)^\top S_{i-1}\bigr)
  +
  \sigma_i \varepsilon_i,
  \label{eq:discrete_score_recursion}
  \\
  G_i
  &=
  \nabla_x\!\bigl(m_i^\theta(X_i)^\top G_{i-1}\bigr)
  +
  \nabla_x\!\left\|
    \frac{m_i^\theta(X_i)-m_i^{\mathrm{ref}}(X_i)}{\sigma_i}
  \right\|^2,
  \label{eq:discrete_pathcost_recursion}
\end{align}
for $i = 1,\ldots,K{-}1$, and the bridge-score and full-horizon adjoint targets at grid point~$t_i$ are
\begin{equation}
  \widehat S_i
  =
  \Bigl(\textstyle\sum_{k=1}^{i}\sigma_k^{2}\Bigr)^{-1} S_i,
  \qquad
  \widehat A_i
  =
  r(X_0)\,\widehat S_i - \tfrac{1}{2} G_i.
  \label{eq:full_horizon_target}
\end{equation}
Using the Bayes bridge score~\eqref{eq:generalized_bayes_bridge_score} instead of Malliavin replaces $\widehat S_i$ by its plug-in form on $(X_0, X_i)$ and leaves the path-cost sweep for $G_i$ unchanged.

\paragraph{Local variant.}
The local estimator~\eqref{eq:lean_target} reuses the same stored rollout with no multi-step sweep.
At each interior index~$i$,
\begin{equation*}
  \widehat A_i^{\mathrm{local}}
  =
  \frac{\widehat V_{i-1}}{\sigma_i} \nabla_x\!\bigl(m_i^{\mathrm{ref}}(X_i)^\top \varepsilon_i\bigr),
  \qquad
  \widehat V_{i-1}
  =
  r(X_0)
  -
  \tfrac{1}{2}
  \sum_{k=1}^{i-1}
  \left\|
    \frac{m_k^\theta(X_k)-m_k^{\mathrm{ref}}(X_k)}{\sigma_k}
  \right\|^2,
\end{equation*}
requiring a single backward pass through the frozen reference step map and no differentiation through the current policy.

\Cref{alg:full_horizon} gives the full-horizon Malliavin procedure.
Each iteration accumulates the exact discrete analogue of the velocity-space regression loss over the interior grid points, and the Bayes and local variants substitute their respective targets into the same rollout.

\begin{algorithm}[t]
  \caption{Full-horizon Malliavin RAM}
  \label{alg:full_horizon}
  \begin{algorithmic}[1]
    \Require parameters~$\theta$ (initialized from $v^{\mathrm{ref}}$),
      reward~$r$, grid $0{=}t_0 < \cdots < t_K{<}1$
    \While{not converged}
      \Statex \hspace{\algorithmicindent}%
        \Comment{\textbf{SDE rollout}}
      \State sample $X_K \sim \mathcal N(0,\,I)$
      \For{$i = K,\ldots,1$}
        \State sample $\varepsilon_i \sim \mathcal N(0,\,I)$
        \State $X_{i-1} \gets m_i^\theta(X_i) + \sigma_i \varepsilon_i$
      \EndFor
      \State $r_0 \gets r(X_0)$
      \Statex \hspace{\algorithmicindent}%
        \Comment{\textbf{Adjoint sweep and loss accumulation}}
      \State $S \gets 0$,  $G \gets 0$,  $L \gets 0$
      \For{$i = 1,\ldots,K{-}1$}
        \State $S \gets
          \nabla_x\!\bigl(m_i^\theta(X_i)^\top S\bigr)
          + \sigma_i \varepsilon_i$
        \State $G \gets
          \nabla_x\!\bigl(m_i^\theta(X_i)^\top G\bigr)
          + \nabla_x\!\left\|
            \frac{m_i^\theta(X_i)-m_i^{\mathrm{ref}}(X_i)}{\sigma_i}
          \right\|^2$
        \State $\widehat S \gets
          \bigl(\textstyle\sum_{k=1}^{i}\sigma_k^{2}\bigr)^{-1} S$
        \State $\widehat A \gets
          r_0\,\widehat S - \tfrac{1}{2} G$
        \State $L \gets L
          + \left\|
            \frac{m_i^\theta(X_i)-m_i^{\mathrm{ref}}(X_i)}{\sigma_i}
          - \mathrm{sg}\!\bigl(\sigma_i\,\widehat A\bigr)
          \right\|^2$
      \EndFor
      \State update $\theta$ with $\nabla_\theta L/(K{-}1)$
    \EndWhile
  \end{algorithmic}
\end{algorithm}

\section{Experimental Details}
\label{sec:experimental_details}

We adapt the general setup of prior work~\citep{liu2025flowgrpo,zheng2025diffusionnft,xue2025awm}.
We post-train Stable Diffusion 3.5 Medium (SD3.5M)~\citep{esser2024scaling} with LoRA (rank~$r=32$, scaling~$\alpha=64$).
We train all models on a cluster of 4 NVIDIA H100 GPUs with 96 GB of memory each.

\paragraph{Training steps.}
Each training step, we draw $48$ prompts, generate $24$ samples per prompt, evaluate the reward on every prompt-image pair, construct $K=8$ noisy training samples per image for which we compute the RAM loss, and perform a single optimizer step.
This means that the effective batch size (across all GPUs) is $48 \times 24 \times 8 = 9216$ per parameter update.
For a fair comparison, we use $48$ prompts per step for all baselines as well (AWM's reference implementation uses $72$); otherwise we retain the hyperparameters and implementation of each original work.

\paragraph{Per-step compute cost.}
\Cref{tab:gpu_hours} reports the average per-step compute cost of each method, measured in GPU-hours (wall-clock time multiplied by the four H100 GPUs we train on).
Per-step costs are comparable across methods, so RAM's step-count efficiency translates directly into wall-clock training time.

\begin{table}[h]
  \centering
  \caption{Average per-step training cost, in GPU-hours (training wall-clock time multiplied by four H100 GPUs).}
  \label{tab:gpu_hours}
  \small
  \begin{tabular}{@{}lccc@{}}
    \toprule
    \textbf{Method} & \textbf{GenEval} & \textbf{PickScore} & \textbf{OCR} \\
    \midrule
    Flow-GRPO & $0.701$ & $0.431$ & $0.498$ \\
    AWM & $0.653$ & $0.329$ & $0.370$ \\
    DiffusionNFT & $0.566$ & $0.254$ & $0.304$ \\
    RAM & $0.666$ & $0.399$ & $0.426$ \\
    \bottomrule
  \end{tabular}
\end{table}

\paragraph{Sampling and guidance.}
We use the default Euler sampler throughout, with $20$ steps during training and $40$ steps at evaluation.
Training uses classifier-free guidance with scale~$2.0$.
At evaluation we use the default SD3.5M scale of~$4.5$ for GenEval and OCR, and a lower scale of~$2.0$ for PickScore.
The lower scale for PickScore reflects its longer training run, during which guidance gets distilled into the model.

\paragraph{Optimizer.}
We use AdamW with learning rate~$3\mathrm{e}{-4}$, weight decay~$0.01$, and no learning-rate warmup.
We lower Adam's second-moment decay~$\beta_2$ from its default of~$0.999$ to~$0.95$.
This is a common choice for short training runs, as the smaller~$\beta_2$ lets the optimizer adapt more quickly to the gradient scale.

\paragraph{Exponential moving average.}
On-policy samples drawn during training come from an exponential moving average (EMA) of the model parameters with decay~$0.9$ and warmup rate~$0.01$.
Prior works use task-dependent EMA configurations; we use a single setting that works uniformly well across all three tasks.
Final evaluation uses an EMA with decay~$0.9$ and no warmup.

\paragraph{Reward coefficient.}
After group-relative normalization, we multiply the rewards by a fixed task-dependent coefficient: $100$ for GenEval and OCR, and $1000$ for PickScore.
Because reward scaling controls the regularization strength in our convention, this coefficient corresponds to $1/\texttt{kl\_weight}$ in methods that instead report an explicit KL-penalty weight.
The values were tuned to maximize the training reward while still scoring well on the image-quality metrics.

\paragraph{Timestep sampling.}
During training we sample the timestep~$t$ from the linear density~$p(t) = 2t$ on~$[0,1]$, biasing toward values near the noise endpoint~$t{=}1$ where generation begins.
This notably enhances training stability over uniform sampling.

\end{document}